\documentclass{article}

 \usepackage[preprint]{neurips_2026}

\usepackage[utf8]{inputenc} % allow utf-8 input
\usepackage[T1]{fontenc}    % use 8-bit T1 fonts
\usepackage{hyperref}       % hyperlinks
\usepackage{url}            % simple URL typesetting
\usepackage{amsfonts}       % blackboard math symbols
\usepackage{nicefrac}       % compact symbols for 1/2, etc.
\usepackage{microtype}      % microtypography
\usepackage{xcolor}         % colors
\usepackage{hyperref}
\usepackage{url}
\usepackage{amsmath,amssymb,amsthm}
\usepackage{amsthm}
\usepackage{algpseudocode}
\usepackage{microtype}
\usepackage{booktabs} % for professional tables
\usepackage{mathtools}
\newtheorem{proposition}{Proposition}
\usepackage{algorithm}
\theoremstyle{plain}
\newtheorem{theorem}{Theorem}[section]

\theoremstyle{definition}

\theoremstyle{remark}
\newtheorem{remark}[theorem]{Remark}
\usepackage{svg}
\usepackage{graphicx}
\usepackage{subcaption}
\usepackage{tabularx}
\usepackage{booktabs}
\usepackage{array}

\title{BayMOTH: Bayesian optiMizatiOn with meTa-lookahead -- a simple approacH}

\author{%
\textbf{Rahman Ejaz}$^{1,2}$\thanks{Correspondence to: \texttt{reja@lle.rochester.edu}} \quad
\textbf{Varchas Gopalaswamy}$^{1,2}$ \quad
\textbf{Ricardo Luna}$^{3}$ \quad
\textbf{Aarne Lees}$^{1,2}$ \\
\textbf{Vineet Gundecha}$^{3}$ \quad
\textbf{Christopher Kanan}$^{4}$ \quad
\textbf{Soumyendu Sarkar}$^{3}$ \quad
\textbf{Riccardo Betti}$^{1,2,5}$ \\
\\
$^{1}$Laboratory for Laser Energetics, Rochester, NY, USA \\
$^{2}$Department of Mechanical Engineering, University of Rochester, Rochester, NY, USA \\
$^{3}$Hewlett Packard Labs, Santa Clara, CA, USA \\
$^{4}$Department of Computer Science, University of Rochester, Rochester, NY, USA \\
$^{5}$Department of Physics, University of Rochester, Rochester, NY, USA
}

\begin{document}

\maketitle

\begin{abstract}
   Bayesian optimization (BO) has for sequential optimization of expensive black-box functions demonstrated practicality and effectiveness in many real-world settings. Meta-Bayesian optimization (meta-BO) focuses on improving the sample efficiency of BO by making use of information from related tasks. Although meta-BO is sample-efficient when task structure transfers, poor alignment between meta-training and test tasks can cause suboptimal queries to be suggested  during online optimization. To this end, we propose a simple meta-BO algorithm that utilizes related-task information when determined useful, falling back to lookahead otherwise, within a unified framework. We demonstrate competitiveness of our method with existing approaches on function optimization tasks, while retaining strong performance in low task-relatedness regimes where test tasks share limited structure with the meta-training set.
\end{abstract}

\section{Introduction}
For realistic applications that requires optimization of an unknown function $f$, high sample evaluation costs often constrains the sampling budget. Examples of such applications arise in hyperparameter tuning of machine learning models \cite{snoek2012practical,zoph2016neural,feurer2019hyperparameter}, in materials design \cite{frazier2015bayesian}, and in physics applications such as in the domain of laser-driven nuclear fusion experiments (ICF) \cite{ejaz2023direct,gundecha2024meta}. For such tasks, Bayesian optimization (BO) \cite{frazier2018tutorial} has proven its utility. However, on the quest of more performant methods i.e. higher sample efficiency, approaches such as meta-BO and lookahead Bayesian optimization (lookahead BO) have emerged. Both improve sample efficiency albeit both utilize different approaches. In order to facilitate discussion, we borrow from the nomenclature of reinforcement learning and refer to the sampling strategy as the sampling policy denoted by $\pi$.
\par In meta-BO various strategies are used to transfer information \cite{bai2023transfer,ma2025toward} drawn from related tasks into components of the optimization methods' pipeline. This enables a more efficient policy $\pi$ to be realized, especially compared with baselines such as random sampling and BO \cite{volpp2019meta}. In contrast to meta-BO, lookahead BO does not rely on knowledge from related tasks and seeks to improve upon BO by recognizing that conventional BO operates under a one-step assumption without consideration of future samples allowed within the sample budget. Polices focused on maximizing the unknown function $f$ in the immediate step are commonly referred to as ``myopic" policies and recent work \cite{wu2019practical,cheon2024earl} 
demonstrates that during online optimization, lookahead policies reach near-optimal regions in fewer samples across a variety of benchmarks, indicating a superior sampling policy. 

\par Combining meta-BO and lookahead BO appears to be a natural extension. However, it has been relatively unexplored, motivating a straightforward framework combining the two approaches. The presented method, which is referred to as BayMOTH, makes use of alternative alternative virtual environments for inferring the presence of shared structure during online optimization between the test task (i.e., unknown function) and source tasks (training functions) provided to it offline \cite{levine2020offline} during meta-training. The shared structure information is used as a task descriptor similar in spirit to usage in multi-task learning \cite{zhao2019recommending} to gate into the pertinent alternative virtual environment which provides meta-data for informing the lookahead based policy.

We demonstrate the competitiveness of BayMOTH on hyperparameter optimization tasks and scientific experiments (ICF experiments) optimization tasks and compare it to incumbent meta-BO methods, multi-fidelity BO, lookahead BO, and conventional BO. The experimental tests are conducted under the context of relevant meta-data being available for meta-BO. Details are provided in Section~\ref{Exps}.\footnote{All datasets and code is provided as supplementary material.}

\par In many practical applications, however, it is often unclear if prior meta-information is useful. BayMOTH is, by construction, more robust to useless prior information, unlike meta-BO. We demonstrate this through targeted experiments, in which to the meta-BO methods the training tasks that are provided are not strongly related to the unknown function $f$. This experimental setup emulates real-world scenarios where the assumed relationship between the meta-information and the unknown function may not hold. In autonomous applications where the value of prior information is unknown, it is vital that the sampling policy must be able to self-assess the value of its priors. In existing BO works, this mismatch has been remarkably understudied, and to the best of  our knowledge has only been considered previously by the robust multi-fidelity BO (MFBO) of \cite{mikkola2023multi} which accepts a MFBO proposal over a single-fidelity BO proposal, if  the MFBO surrogate passes some conditions.

\par Moreover, as becomes evident in Section~\ref{BayMOTH}, BayMOTH is a relatively simple algorithm when compared with popular meta-BO approaches, which employ reinforcement learning (RL) to incorporate meta-information into their respective sampling policies. It is well known that training RL based policies can be challenging due to the deadly triad of RL \cite{sutton1998reinforcement}. To counteract such effects, state-of-the-art (SOTA) meta-BO methods often include auxiliary tasks for stable policy training such as in the work of \cite{maraval2023end}, which complicates the training process. While simplicity is not sufficient in isolation as a merit-able attribute, the complexity in algorithms and training procedure can preclude usage from practitioners outside of ML/RL expertise. For such practitioners sample efficient black-box optimization can be of great interest due to high cost of evaluations for their respective domains such as for scientific experiments.
Therefore, the competitive performance and simplicity of BayMOTH, we deem useful to practitioners. 

\textbf{Contributions.}
Our main contributions are summarized as follows:
\begin{itemize}
    \item We highlight brittleness in meta-BO under source task mismatch.
    \item We propose a simple meta-Bayesian optimization algorithm (BayMOTH), that remains robust under source-task mismatch and leverages related-task structure without relying on RL pretraining.
    \item BayMOTH unifies lookahead and meta-BO perspectives in a simple algorithm.
    \item We show empirically that BayMOTH performs competitively against state-of-the-art meta-BO methods and strong established baselines, and we situate its performance relative to these approaches.
\end{itemize}

\section{Background}
\label{background}

\subsection{Problem formulation}

We wish to find the optimum of an unknown expensive to evaluate black-box function $f:\mathcal{D} \to \mathbb{R}$ on a known domain $\mathcal{D} \subseteq \mathbb{R}^d$, within a sample budget $T$. A policy $\pi$ chooses points to sample $x \in \mathcal{D}$ from $f(x)$ in a sequential manner resulting in a evaluation history $\mathcal{H}_{t} = \{ x_i, f(x_i) \}_{i=1}^{t}$ after 
$t$ sample steps. A sampling policy $\pi$ is considered superior if the simple regret $r(t) := \max_{x \in \mathcal{D}} f(x) - \max_{1 \le i \le t} f(x_i)$ is minimized in fewer sample steps.

\subsection{Bayesian Optimization}

In Bayesian optimization (BO) \cite{shahriari2015taking,frazier2015bayesian} a probabilistic surrogate model, commonly a \textit{Gaussian Process} ($\mathcal{GP}$) \cite{williams2006gaussian} is used in conjunction with an \textit{acquisition function} (AF) for choosing sample locations that potentially maximize $f(x)$. Given an evaluation history $\mathcal{H}_{t} = \{ x_i, f(x_i) \}_{i=1}^{t}$, the $\mathcal{GP}$  posterior is characterized by the mean $\mu_t(x)=\mathbb{E}[f(x)|\mathcal{H}_{t}]$ and variance $\sigma_t=\mathbb{V}[f(x)|\mathcal{H}_{t}]^{\frac{1}{2}}$, which have closed analytic forms \cite{williams2006gaussian}. Utilizing the $\mathcal{GP}$ posterior, an acquisition function provides a score value for $x \in \mathcal{D}$ taking also into account exploitation-exploration tradeoffs. A widely used AF is \textit{Expected Improvement} (EI) \cite{jones1998ei} which at sample step \textit{t} is given as:
\[
\mathrm{EI}_{t+1}(x) = \mathbb{E}[\mathrm{max}(f(x) - \max_{1 \le i \le t} f(x_i),0)]
\]

The BO policy, $\pi_{\mathrm{BO}} \to x_{t+1}$ is $\operatorname*{argmax}_{x \in \mathcal{D}} \mathrm{EI}_{t+1}(x)$.

\subsection{Meta-Bayesian Optimization}
\label{meta-bo background}

In contrast to BO, where only the evaluation history informs the policy, meta-BO makes use of information available from related functions which is typically available in the form of sample-evaluation pairs as shown: 
\[
\begin{aligned}
\mathcal{D}_{\tau} &= \{(x_1, f_{\tau}(x_1)), \dots, (x_N, f_{\tau}(x_N))\},
\end{aligned}
\]
from any arbitrary policy forming task-specific datasets $\mathcal{D}_1, \mathcal{D}_2, \dots \mathcal{D}_{\tau}$ that collectively constitute the meta-dataset ($\mathcal{D_{\text{meta-train}}}$).
Approaches on how the meta-dataset is utilized differs among different approaches. Prior work has focused on the surrogate model component of BO, such as through ranking-weighted Gaussian process ensembles \cite{feurer2018scalable} and multi-task Gaussian processes, or on the acquisition function component of BO, for example via transfer acquisition functions \cite{wistuba2018scalable}, whereas the work of \cite{maraval2023end} focuses on both components in an end-to-end framework. A comprehensive survey of meta-BO can be found in \cite{bai2023transfer} and \cite{ma2025toward}.

Here, we expand on some foundational meta-BO methods. Generally, these methods learn a general sampling rule (policy) that at test-time require less sample steps compared with non-meta-BO. This is achieved by incorporating information from the meta-dataset obtained from related functions. Examples of related functions may include synthetic families of a well defined function such as sine waves with varying phase and amplitude over a shared domain \cite{finn2017model}, or in more abstract settings can be in the context of hyperparameter optimization (HPO) where model accuracies as a function of different hyperparameter configurations and search spaces are available for multiple classification problems \cite{arango2021hpo}.
\par The approach of \cite{volpp2019meta} (known as MetaBO) makes use of a $\mathcal{GP}$ surrogate model similar to the traditional BO setup, but makes use of a \textit{neural acquisition function} (NAF) in lieu of a classical AF. The meta-dataset is used to train the NAF with the Proximal Policy Optimization (PPO) RL algorithm where NAF takes as input $[\mu_t(x),\sigma_t(x),x]$ for the state-action representation at each sample step to output a score-value on a discrete set of points on the optimization domain. During PPO training the current policy rolls-out $\mathcal{H}_{t} = \{ x_i, f(x_i) \}_{i=1}^{t}$ from $f$ available in $\mathcal{D_{\text{meta-train}}}$, and a respective $\mathcal{GP}_t$  provides $\mu_t(x)$ and $\sigma_t(x)$ to the NAF. In the training loop the input information for NAF along with the auxiliary information (i.e. reward) needed for PPO is recorded for an episode, batches of which are used for the NAF policy update. Conceptually, after meta-training, NAF has learned how to use $[\mu_t(x),\sigma_t(x),x]$ for providing a high-score for promising sample locations over a variety of functions. As mentioned by \cite{volpp2019meta} the inclusion of $x$ as an input enables it to identify shared structure between functions used during offline meta-training and functions it will encounter in usage.

\par Similar to MetaBO \cite{volpp2019meta}, FSAF \cite{hsieh2021reinforced} also uses $[\mu_t(x),\sigma_t(x)]$ from a $\mathcal{GP}$ as part of an embedding for the state-action representation. For the acquisition function a Bayesian approach is used to learn $N$ deep-Q-networks (DQNs) in a MAML bi-level framework \cite{finn2017model,yoon2018bayesian}, such that from the collection of $f \in \mathcal{D}_{\text{meta-train}}$ a generalized policy is learned that is well suited for one-shot adaption to the unknown function $f$, by making gradient updates to the DQN using $\mathcal{H}_{t} = \{ x_i, f(x_i) \}_{i=1}^{t}$.

\subsection{Multi-fidelity Bayesian optimization (MFBO)}
MFBO extends standard BO by allowing (during online sampling) a query be made to the expensive test task or cheaper auxiliary information sources, by utilizing a joint surrogate over all fidelities in an augmented space and an acquisition function that trades off informativeness vs. cost across fidelities. Its benefit relies on the auxiliary fidelities being sufficiently informative about the test task. However, when this assumption fails, MFBO can waste budget on misleading sources and even under perform single-fidelity BO under the same evaluation budget. The robust MFBO of \cite{mikkola2023multi} addresses this issue through a wrapper for GP-based MFBO methods that only accepts a multi-fidelity proposal when reliability conditions are met and otherwise falls back to a conservative single-fidelity alternative. We draw attention that BayMOTH is related to this perspective at a high-level, as both are motivated by the principle that source information should be exploited only when it is determined useful, with the sampling policy remaining robust under source-test task mismatch. However, both approaches are fundamentally different. Furthermore, in the MFBO paradigm, auxiliary source information is revealed only online through sequential queries to that fidelity, with the objective being more focused on reducing costly evaluations on the target task. In contrast, in BayMOTH and meta-BO more broadly, source-task information is available a priori through meta-training, therefore, source task information is already available before online optimization begins.

\subsection{Lookahead Bayesian Optimization}

Lookahead BO polices aim to improve upon BO by taking into consideration the allowance of future evaluations within the sample budget, with the assumption that an exploration sample taken at a current sample step $t$ can reveal promising locations for finding the optimum location in future samples. 
\begin{remark}
    The meta-BO methods mentioned in Section ~\ref{meta-bo background} also utilize $t/T$ as part of their state-action representation which can influence future sample points. However, this approach remains myopic, as, unlike lookahead methods, it does not explicitly simulate possible future sample outcomes.
\end{remark}

The practical two-step lookahead acquisition function proposed by \cite{wu2019practical}, known as 2-OPT, demonstrates that even limited-horizon (two-step) lookahead can yield meaningful performance gains over standard one-step BO. The 2-OPT acquisition function itself consists of two  $\mathrm{EI}$  acquisition functions each pertaining to a stage in the two-stage process, The first $\mathrm{EI}$ component ($\mathrm{EI}_0$) considers $\mathcal{GP}_0$ from $\mathcal{H}_{t} = \{ x_i, f(x_i) \}_{i=1}^{t}$ , whereas $\mathrm{EI}_1$ considers $\mathcal{GP}_1$ from $\mathcal{H}_{t_{\mathrm{sim}}} = \mathcal{H}_t \cup \{x_{t+1},f(x_{t+1}) \sim \mathcal{GP}_0\}$. Thus, the 2-OPT acquisition function takes the form 
\begin{equation*}
\mathrm{2\text{-}OPT}(x_{t+1})
\;\approx\;
\mathrm{EI}_0(x_{t+1})
\;+\;
\mathbb{E}\!\left[
    \max_{x_{t+2}} \mathrm{EI}_1(x_{t+2})
\right].
\end{equation*}

Where, the expectation is taken under $\mathcal{GP}_0$. The BayMOTH policy is based on the 2-OPT acquisition function, its details are presented in section \ref{BayMOTH}.

Multi-step lookahead aims to maximize long-horizon optimization performance. However, due to scalability issues, emulating multi-step BO has been considered instead and is shown to be performant in the work by \cite{cheon2024earl} by making use of an Attention-Deep-Sets encoder for projecting state information onto a latent space suitable for a RL framework.

\section{BayMOTH}
\label{BayMOTH}

\begin{figure}[t]
    \centering
    \includegraphics[width=\linewidth]{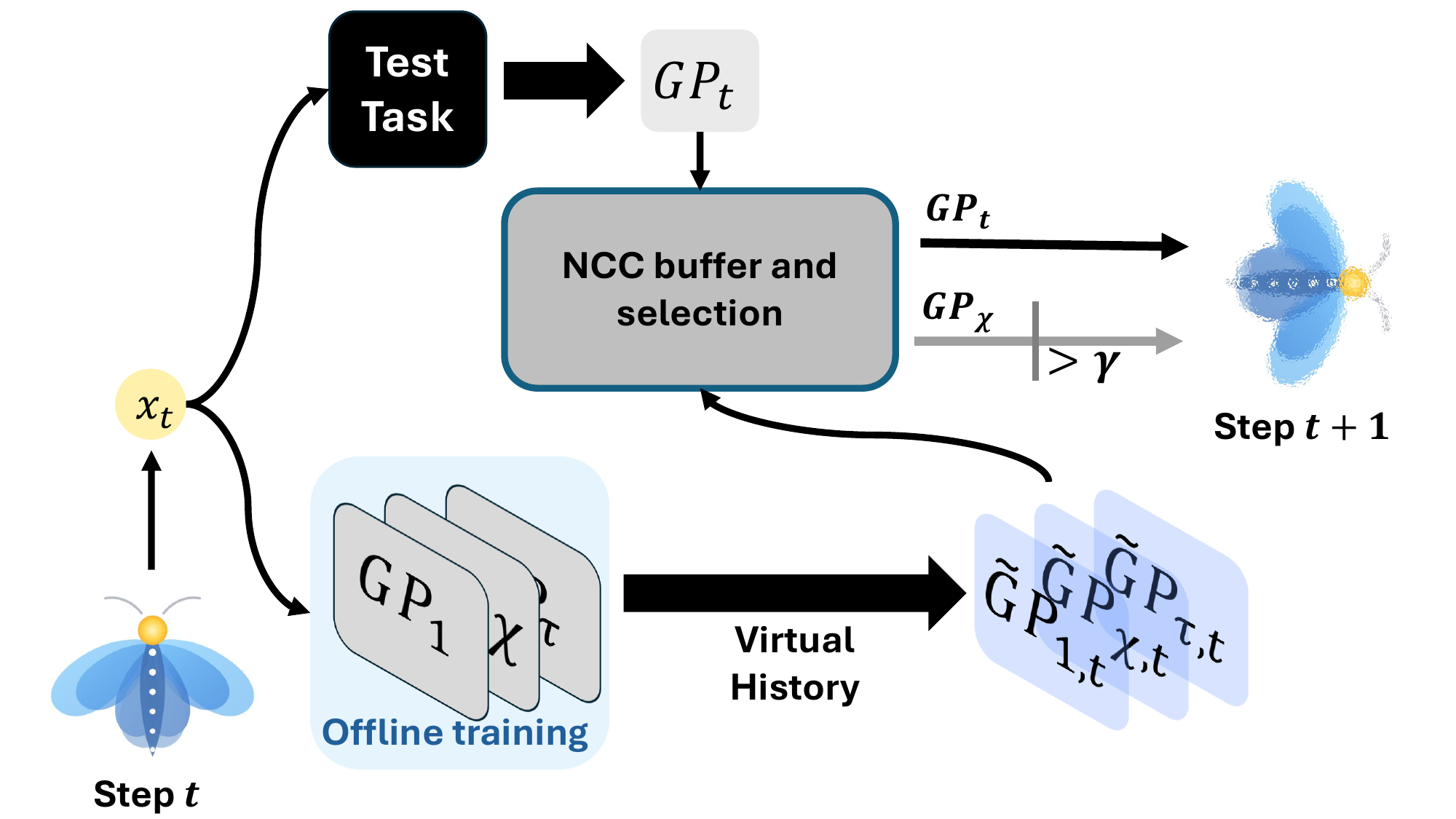}
    \caption{An overview of BayMOTH}
    \label{fig:baymoth}
\end{figure}

In this section we introduce BayMOTH, a meta-BO policy that extends the 2-OPT acquisition function by incorporating meta-information through a well-defined selection rule, and fallbacks to 2-OPT when such information is deemed uninformative. A key feature of BayMOTH is that it is capable of autonomously deciding when the meta-information is useful. We first describe the inclusion procedure of meta-information based on shared structure identification with the related functions, and then present the BayMOTH policy. Finally, we discuss simplifications made to the policy for practical usage. The full algorithm for BayMOTH is presented in Section~\ref{app:baymoth_algo}

\subsection{Shared structure formulism}
 A common peril for meta-learning based approaches is the memorization problem \cite{yin2019meta}. The memorization problem is a form of task overfitting where the meta-learner overfits to the finite set of provided training tasks\footnote{In the meta-learning literature, the terms \emph{related tasks}~\cite{finn2017model} and \emph{related functions} are often used interchangeably.}
 by implicitly encoding task-specific solutions, rather than learning transferable attributes that enable good performance on unseen tasks. This form of overfitting is subtle as the meta-learning based model can still be performant if the unseen task is similar to the training tasks. This concept is formalized in \cite{yin2019meta}. Meta–BO methods are also susceptible to memorization, as strong meta-learned surrogates or acquisition policies can achieve low meta-training regret by encoding task-specific structure present in the training set, rather than learning robust exploration/exploitation mechanisms that generalize to unseen tasks. As a result, when task diversity is limited, performance gains may diminish when evaluated on unseen objectives or complex tasks. A mitigating strategy is to increase the number of training tasks \cite{volpp2019meta}, though, its applicability is limited in scientific or HPO settings where datasets are expensive to generate and the available meta-data is inherently fixed. Other strategies include restricting model capacity and \cite{yin2019meta} propose a meta-regularizer based on information theory.
 
 \par For robustness against memorization and explicit usage of only relevant meta-information, BayMOTH makes use of alternative virtual environments. A alternative virtual environment for each dataset in $\mathcal{D}_{\text{meta-train}}$ is constructed by fitting a $\mathcal{GP}$ to $\mathcal{D}_{\chi} = \{(x_i, f_{\chi}(x_i))\}_{i=1}^{n}$, where $\mathcal{D}_{\chi} \in \mathcal{D}_{\text{meta-train}}$. Consequently, given $\tau$ related tasks available for meta-training, i.e., $\{\mathcal{D}_1, \ldots, \mathcal{D}_{\chi}, \ldots, \mathcal{D}_{\tau}\}$, this procedure induces a collection of $\tau$ task-specific alternative virtual environments containing $\{\mathcal{GP}_1, \ldots, \mathcal{GP}_{\chi}, \ldots, \mathcal{GP}_{\tau}\}$.

At optimization step $t+1$ on an unknown task, the observed sample history
$\mathcal{H}_{t} = \{(x_i, f(x_i)\}_{i=1}^{t}$ is reused to query in each virtual
environment yielding virtual sample
histories
$\tilde{\mathcal{H}}_{t}^{\chi} = \{(x_i, \tilde{f}_{\chi}(x_i)\}_{i=1}^{t}$,
where $\tilde{f}_{\chi}(x_i) = \mathcal{\mu}_{\chi}(x_i)$ from $\mathcal{GP}_{\chi}(x_i)$. Each virtual history
$\tilde{\mathcal{H}}_{t}^{\chi}$ is then used to fit a $\tilde{\mathcal{GP}}_{\chi,t}$ using the same kernel and
hyperparameter configuration as the surrogate $\mathcal{GP}_{t}$ fitted to
the unknown task. The resulting surrogates are compared to
$\mathcal{GP}_{t}$ using normalized cross-correlation (NCC) between their
predictive mean functions, defined as:
\[
\begin{aligned}
\mathrm{NCC}\!\left(\mu_t, \tilde{\mu}_{\chi,t}\right)
&=
\frac{
\sum_{j=1}^{M}
(\mu_t(x_j) - \bar{\mu}_t)
(\tilde{\mu}_{\chi,t}(x_j) - \bar{\tilde{\mu}}_{\chi,t})
}{
\left\|\mu_t - \bar{\mu}_t\right\|_2\,
\left\|\tilde{\mu}_{\chi,t} - \bar{\tilde{\mu}}_{\chi,t}\right\|_2
},
\end{aligned}
\]

where $\mu_{t}(\cdot)$ and $\tilde{\mu}_{\chi,t}(\cdot)$ denote the predictive
means of $\mathcal{GP}_{t}$ and $\tilde{\mathcal{GP}}_{\chi,t}$,
respectively, evaluated on a shared reference set
$\{x_j\}_{j=1}^{M}$, and $\bar{\mu}_{t}$ and
$\bar{\tilde{\mu}}_{\chi,t}$ denote their empirical means. The meta-training
task achieving the highest NCC score is selected as exhibiting the most shared
structure with the unknown task. This score is used as a task descriptor to gate into the pertinent alternative virtual environment through the selection rule, whose corresponding $\mathcal{GP}$ now denoted as $\mathcal{GP}_{\chi^*}$ is used in Eq~\ref{eq:BayMOTH}. If each alternative virtual environment's NCC score is below a threshold, BayMOTH fallbacks to 2-OPT.
\par Normalized cross-correlation is well suited for this comparison, as it is invariant to affine scaling of function values and emphasizes structural alignment, making it robust to differences in magnitude and noise levels even in high-dimensional input spaces. While this procedure requires fitting $\tau$ auxiliary Gaussian processes at each optimization step, the computational overhead remains modest relative to the cost of unknown function evaluations in many real-world usage scenarios. Furthermore, the resulting simplicity and interpretability of the task-selection mechanism provide a favorable trade-off in practice.

\subsection{BayMOTH policy}
Given data $\mathcal{H}_{1:t}$ from the unknown function $f$,
where,
\begin{equation*}
\mathcal{H}_{1:t}
= \Big\{\,(\mathbf{x}_1, f_1), (\mathbf{x}_2, f_2), \dots, (\mathbf{x}_t, f_{t})\,\Big\},
\end{equation*}

the next sample location is provided from BayMOTH according to its policy $\pi_{\mathrm{BayMOTH}}$ which is given by:

\begin{equation}
\label{eq:BayMOTH}
\arg\max_{x_{t+1}}
\Big[
\Lambda_0(x_{t+1})
+
\mathbb{E}_{\mathcal{GP}_{\chi^*}}
\Big[
(1-\alpha)\max_{x_{t+2}}\Lambda_1(x_{t+2})
+
\alpha\,\Omega
\Big]
\Big]
\end{equation}

Here, $\Lambda_0$ is the $\mathrm{EI}$ acquisition function considering the $\mathcal{GP}$ posterior from data $\mathcal{H}_{1:t}$, and $\Lambda_1$ is $\mathrm{EI}$ considering the $\mathcal{GP}$ posterior from data $\mathcal{H}_{1:{t}_{\mathrm{sim}}}$. Where , $\mathcal{H}_{1:{t}_{\mathrm{sim}}}$ contains the simulated sample evaluation at $x_{t+1}$,
\begin{equation}
\mathcal{H}_{1:t_{\text{sim}}}
= \mathcal{H}_{1:t} \cup \Big\{\,x_{t+1},\, f(x_{t+1}) \sim \mathcal{GP}_{\chi^{*}} \,\Big\}.
\end{equation}
The $\Omega$ term is a greedy improvement component defined as
\(
\Omega
=
\max\!\big(0,\,
f(x_{t+1}) - f_{\max}^{(t)}
\big)
\)
where
\(
f_{\max}^{(t)}=\max_{1\le i\le t} f_i.
\)
The term in the expectation is interpreted as a convex combination of a meta-information derived future expected improvement term with an immediate greedy improvement term. By interpolating these terms as such, we are balancing horizon planning with direct exploitation, with both terms being informed by meta-information. The behavior of BayMOTH with this balancing is elaborated through an ablation study on $\alpha$ in Appendix~\ref{app:ablation_alpha}.

\subsection{BayMOTH simplification}
The second term in Eq.\ref{eq:BayMOTH} is an expectation under $\mathcal{GP}_{\chi^{*}}$, considering the randomness of $\mathcal{GP}_{\chi^{*}}$ as $\mathcal{P}(\omega)$ we define this term as,

\[
\mathbb{J}(\omega) \vcentcolon=  \mathbb{E}_{\omega \sim \mathcal{P}}\Big[(1-\alpha)\max_{x_{t+2}}\Lambda_1(x_{t+2};\omega)+\alpha\Omega(\omega)\Big]
\]
and,
\[
\mathbb{Z}(\omega) \vcentcolon=(1-\alpha)\max_{x_{t+2}}\Lambda_1(x_{t+2};\omega) + \alpha \Omega(\omega)
\]
then,
\[
\mathbb{J}(\omega) = \int\mathbb{Z}(\omega)\mathcal{P}(\omega)\mathrm{d}\omega 
\]
\[
= \int \Big((1-\alpha)\max_{x_{t+2}}\Lambda_1(x_{t+2};\omega) + \alpha \Omega(\omega)\Big)\mathcal{P}(\omega)\mathrm{d}\omega
\]

as this is intractable analytically, we use a Monte Carlo estimator

\begin{equation}
\hat{J}_M
\;=\;
\frac{1}{M}
\sum_{m=1}^M
\mathbb{Z}(\omega^{(m)}),
\qquad
\omega^{(m)} \sim \mathrm{GP}_{\chi^{*}} \;\text{i.i.d.}
\end{equation}

This estimator is unbiased and converges to $\mathbb{J}$ as $M \to \infty$ by the law of large numbers, with variance $\mathbb{V}(\hat{J}_M) = \sigma^{2}/M$ where $\sigma^{2}$ is the variance of $\mathbb{Z}$. These properties are detailed in Appendix~\ref{app:proof}.
\par Although the variance of the estimator decays at rate $(O(1/M))$, in practice, relatively small number of samples is sufficient for estimator accuracy \cite{bishop2023deep}. As $GP_{\chi^{*}}$ posterior depends on the samples drawn from $\mathcal{D}_{\chi}$, the variance on the random variable $\mathbb{Z}(\omega)$ depends directly on how well $\mathcal{D}_{\chi}$ covers the underlying function over the domain of interest. When the domain is densely sampled,  the posterior variance of $GP_{\chi^{*}}$ is typically small throughout the domain, which in turn implies that $\sigma^{2}$ is also small. Since the estimator has variance $\sigma^{2}/M$, a modest number of samples is a sufficient approximation for small $\sigma^{2}$.
\par In instances where the domain is sparsely sampled for $\mathcal{D}_{\chi}$, BayMOTH is still performant as a relative ranking of candidate points is the quantity of interest over accurately estimating $\mathbb{Z}$. Moreover, the stochasticity in $\mathbb{Z}$ can even encourage exploration. In all experiments we set $M=5$, which we found to provide a good trade-off between computational cost and performance. A small ablation in Appendix~\ref{app:ablation_m} shows that larger M yields only marginal performance gains.

\subsection{Informal Theoretical Analysis}

BayMOTH is built on a gated two-branch architecture. When the inferred task-relatedness is low, it reverts to the 2-OPT acquisition; when the inferred task-relatedness is sufficiently high ($\gamma$ dependent), it uses a meta-informed lookahead branch that modifies the horizon planning term through the selected alternative virtual environment and the greedy improvement component. Accordingly, rather than comparing BayMOTH to a universally optimal policy, we analyze it relative to an idealized oracle version of itself. This oracle comparator uses the correct branch, correct source-task (when applicable), together with the exact lookahead term, and in particular reduces to 2-OPT on rounds where fallback is the appropriate choice.

To formalize this, let
\[
f^\star := \max_{x \in D} f(x),
\qquad
r_T^{\mathrm{BM}} := f^\star - \max_{1 \le t \le T} f(x_t^{\mathrm{BM}})
\]
denote the simple regret of BayMOTH after \(T\) evaluations.

At round \(t\), let \(\mathcal X_t\) denote the candidate set considered by the optimization routine. We write the fallback 2-OPT acquisition as
\[
A_t^{\mathrm{2OPT}}(x)
:=
\Lambda_{0,t}(x)
+
\mathbb E_{GP_t}\!\left[\max_{x' \in \mathcal X_t} \Lambda_{1,t}(x';x)\right],
\]
and let \(A_t^{\mathrm{meta},\dagger}(x)\) denote the \emph{oracle} meta-informed acquisition, defined as the exact meta-lookahead branch evaluated using the correct source task and exact expectation. Let \(g_t^\dagger \in \{0,1\}\) denote the oracle branch selection decision, where \(g_t^\dagger = 0\) selects the fallback (2-OPT) and \(g_t^\dagger = 1\) selects the oracle meta branch. The resulting oracle-gated comparator is
\[
A_t^\dagger(x)
=
(1-g_t^\dagger) A_t^{\mathrm{2OPT}}(x)
+
g_t^\dagger A_t^{\mathrm{meta},\dagger}(x).
\]
Let
\[
x_t^\dagger \in \arg\max_{x \in \mathcal X_t} A_t^\dagger(x),
\qquad
r_T^\dagger := f^\star - \max_{1 \le t \le T} f(x_t^\dagger)
\]
denote the corresponding oracle-gated simple regret.

BayMOTH itself uses a NCC based branch selection decision \(g_t \in \{0,1\}\). When \(g_t=0\), it optimizes the same fallback 2-OPT branch. When \(g_t=1\), it optimizes a Monte Carlo approximation of the meta branch based on the selected alternative virtual environment. We denote the acquisition actually optimized by BayMOTH as \(\hat A_t^{\mathrm{BM}}(x)\).

We make the following assumptions.

\paragraph{Assumption 1 (bounded objective).}
There exists \(B>0\) such that
\[
0 \le f(x) \le B, \qquad \forall x \in D.
\]

\paragraph{Assumption 2 (meta-vs-fallback branch disagreement probability).}
Let
\[
\mathcal E_t^{\mathrm{route}} := \{g_t \neq g_t^\dagger\}
\]
denote the event that BayMOTH disagrees with the oracle branch selection rule at round \(t\). Assume
\[
\Pr(\mathcal E_t^{\mathrm{route}}) \le p_t.
\]

\paragraph{Assumption 3 (within-meta branch virtual-environment selection error).}
For each round $t$, let $\chi_t^\dagger$ denote the oracle virtual-environment index, i.e., the source task that the oracle meta branch would use when meta-routing is appropriate. Let $\hat{\chi}_t$ denote the virtual-environment index selected by BayMOTH's NCC rule on rounds where $g_t=1$.

Let
\[
A_t^{\mathrm{meta}}(x;\chi)
\]
denote the \emph{common exact meta-acquisition template} evaluated using alternative virtual environment $GP_\chi$ and with the expectation computed exactly (that is, without Monte Carlo approximation). Then, on rounds where both BayMOTH and the oracle use the meta branch,
\[
A_t^{\mathrm{meta},\dagger}(x)=A_t^{\mathrm{meta}}(x;\chi_t^\dagger),
\qquad
A_t^{\mathrm{meta},\star}(x)=A_t^{\mathrm{meta}}(x;\hat{\chi}_t),
\]
and we assume
\[
\sup_{x\in\mathcal X_t}
\left|
A_t^{\mathrm{meta}}(x;\hat{\chi}_t)
-
A_t^{\mathrm{meta}}(x;\chi_t^\dagger)
\right|
\le \eta_t.
\]

If $(g_t,g_t^\dagger)\neq(1,1)$, this assumption is not invoked: the case
$g_t\neq g_t^\dagger$ is handled by Assumption 2, and when
$g_t=g_t^\dagger=0$, both methods use the same fallback $2$-OPT branch.

\paragraph{Assumption 4 (Monte Carlo approximation error).}
Let \(\hat J_{t,M}(x)\) denote the Monte Carlo estimator used by BayMOTH in the meta branch with \(M\) samples.
Assume
\[
\mathrm{Var}(\hat J_{t,M}(x)) = \frac{\sigma_t^2(x)}{M},
\]
and define
\[
\sigma_t := \sup_{x \in \mathcal X_t} \sigma_t(x).
\]

\paragraph{Assumption 5 (acquisition-to-value stability).}
There exists \(L_t > 0\) such that for any two acquisitions \(A_t\) and \(\tilde A_t\) on \(\mathcal X_t\),
\[
\sup_{x \in \mathcal X_t} |A_t(x) - \tilde A_t(x)| \le \varepsilon
\quad\Longrightarrow\quad
f(x_{\tilde A_t}) - f(x_{A_t}) \le L_t \varepsilon,
\]
where \(x_{A_t} \in \arg\max_{x \in \mathcal X_t} A_t(x)\) and similarly for \(x_{\tilde A_t}\).

\begin{proposition}[Informal simple-regret decomposition for BayMOTH]
\label{regret_bound}
For any sequence \(\{\delta_t\}_{t=1}^T\) with \(\delta_t \in (0,1)\), with probability at least
\[
1 - \sum_{t=1}^T |\mathcal X_t| \delta_t - \sum_{t=1}^T p_t,
\]
the simple regret of BayMOTH satisfies
\[
r_T^{\mathrm{BM}}
\le
r_T^\dagger
+
\max_{1 \le t \le T}
g_t^\dagger
L_t
\left(
\eta_t + \frac{\sigma_t}{\sqrt{M\delta_t}}
\right).
\]
Moreover,
\[
\mathbb E[r_T^{\mathrm{BM}}]
\le
\mathbb E[r_T^\dagger]
+
\max_{1 \le t \le T}
g_t^\dagger
L_t
\left(
\eta_t + \frac{\sigma_t}{\sqrt{M\delta_t}}
\right)
+
B\left(
\sum_{t=1}^T |\mathcal X_t|\delta_t + \sum_{t=1}^T p_t
\right).
\]
\end{proposition}

\begin{proof}
[Proof sketch]
The proof sketch is given in \ref{proof}
\end{proof}

This proposition should be interpreted as a regret decomposition rather than as a sharp problem-dependent theorem. It shows that BayMOTH tracks an oracle-gated version of its own two-branch architecture. In particular, BayMOTH differs from it only through three BayMOTH-specific sources of error: meta-vs-fallback branch disagreement, mismatch between the selected alternative virtual environment and the oracle source model on meta-routed rounds, and Monte Carlo approximation error in the lookahead term.

\section{Experiments}
\label{Exps}
We demonstrate the competitiveness and position the performance of BayMOTH with existing meta-BO approaches that have demonstrated (MetaBO \cite{volpp2019meta} and NAP \cite{maraval2023end}), with myopic and non-myopic BO (GPBO-EI and 2-OPT \cite{wu2019practical}) and additionally compare with random search (RS). We evaluate across a variety of optimization tasks under a meta-learning context. For hyperparameter optimization tasks we use the HBO-B benchmark \cite{arango2021hpo} and the HPOBench dataset \citep{eggensperger2021hpobench}. For scientific experiments optimization tasks we work in the domain of synthetic ICF experiments. Our results are reported in terms of simple regrets as a function of optimization step, shown in the corresponding figures is the $<>$ line and $\pm\sigma$ shaded region from optimization trajectories obtained over 100 runs (unless stated otherwise) from randomly initialized starting locations on the test task.

\subsection{Synthetic ICF Experiments}
\label{icf_exps}
\begin{figure*}[t]
    \centering
    \includegraphics[width=\textwidth]{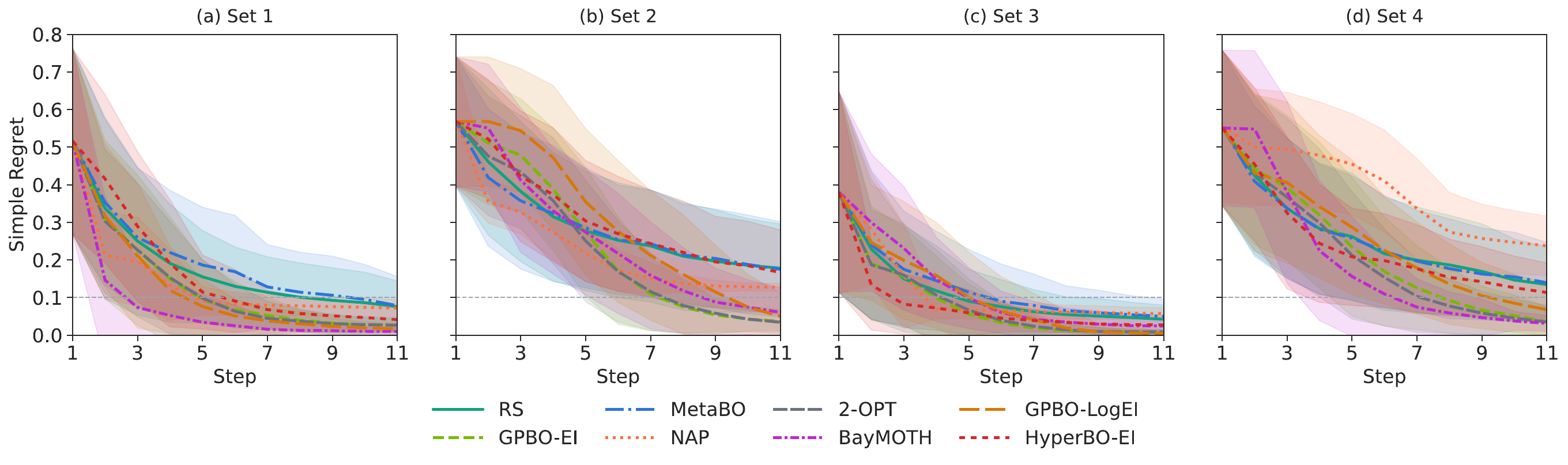}
    \caption{Optimization trajectories across relatedness settings (Sets 1--4) for synthetic ICF experiments. Sample budget is restricted to 10, reflecting the realistic sample budget for ICF.}
    \label{fig:regret_all_sets}
\end{figure*}
\par Inertial confinement fusion (ICF) is a laser-driven scheme for nuclear fusion and offers a path to a clean and abundant energy source \cite{atzeni2004physics}. A key hindrance towards the realization of optimal experimental designs is the limited predictive capability via physics modeling tools, therefore, approaches that use direct experimental feedback for optimization are increasingly becoming popular \cite{ejaz2024optimizing}. However, as ICF experiments are extremely costly and allocation of facility time is competitive, we first demonstrate the practicality of BayMOTH for this purpose through synthetic ICF optimization experiments, similar to the studies of \cite{gundecha2024meta,gutierrez2024explainable}. In addition, we study the behavior of BO approaches under progressively increasing distribution shift between the test task and the source tasks used in meta-training. Specifically, we consider four evaluation regimes spanning low, moderate, medium and high degrees of relatedness to the source tasks. This experimental design is intended to emulate realistic scenarios in which the assumed physics-based models are only weakly related, or potentially unrelated, to the experimental response. Details on task generation, source/test tasks sets' relatedness determination and visualizations of the sets are provided in Appendix \ref{app:icf_exp_details}.

\textbf{High Relatedness (Set 1):} In this setting, BayMOTH reaches the vicinity of the optimum region (0.1 regret) rapidly during the initial iterations (3 steps), suggesting that it is most effective at leveraging specific information from the source tasks, followed by NAP, which also exhibits quicker reduction in regret initially. Here, we highlight that GPBO-EI  is also effective at minimizing the regret within the sample budget, indicating its robustness as a general sampling policy.
\par\textbf{Medium/Moderate Relatedness (Set 2/3):}  In this regime, no clear advantage of meta-BO can be observed. On the contrary for set 2 it is observed that neither NAP or MetaBO are able to obtain less than 0.1 regret ($90\%$ optimal value). This is possibly due to the test task optimum not sufficiently overlapping with that of any individual source task. In contrast, BayMOTH exhibits more robust behavior, and is less susceptible to over reliance on source tasks. For test 3, all methods obtain less than 0.1 regret, with NAP reaching below that level the quickest. We attribute strong performance from all approaches due the optimum region covering a large portion of the domain. 
\par\textbf{Low Relatedness (Set 4):} In this setting we observe both MetaBO and NAP to exhibit suboptimal performance, with both failing for regret reduction. This is suggestive of egregious samples being proposed and is indicative of the memorization problem for meta-BO. As MetaBO and NAP, rely on policies learned from meta-training data, in the low-relatedness regime, they continue to exploit uninformative priors on the policy, and are unable to identify task mismatch leading to downgraded sampling performance. In contrast, BayMOTH is designed to disregard meta-information once it is identified as uninformative, falling back to the 2-OPT policy. As a result, it is able to effectively minimize regret, demonstrating robustness even in this low-relatedness setting. While BayMOTH may be susceptible to memorization effects early in the optimization trajectory, when limited observations make it difficult to discard alternative virtual environments, it nevertheless remains reliable for regret reduction.

\par\textbf{General observations from all Sets:} From these tests it is observed that BayMOTH is performant under a variety of distribution shift scenarios. Whereas, NAP can learn a strong sampling policy when provided with appropriate related tasks during meta-training, as can been deduced from Figure~\ref{fig:regret_all_sets}a and  Figure~\ref{fig:regret_all_sets}c, MetaBO is less competitive on these tests, potentially due to the limited number of available meta-training tasks (10 tasks). The weaker performance could also be attributed to generally a weaker sampling policy for MetaBO, which is consistent with the observations of \cite{maraval2023end} where MetaBO rarely outperformed GPBO-EI. We further investigate the sensitivity of provided training tasks on test task performance in Appendix~\ref{app:sensitivity_details}.

\subsection{Hyper Parameter Optimization Experiments}
\begin{figure*}[t]
    \centering
    \includegraphics[width=\textwidth]{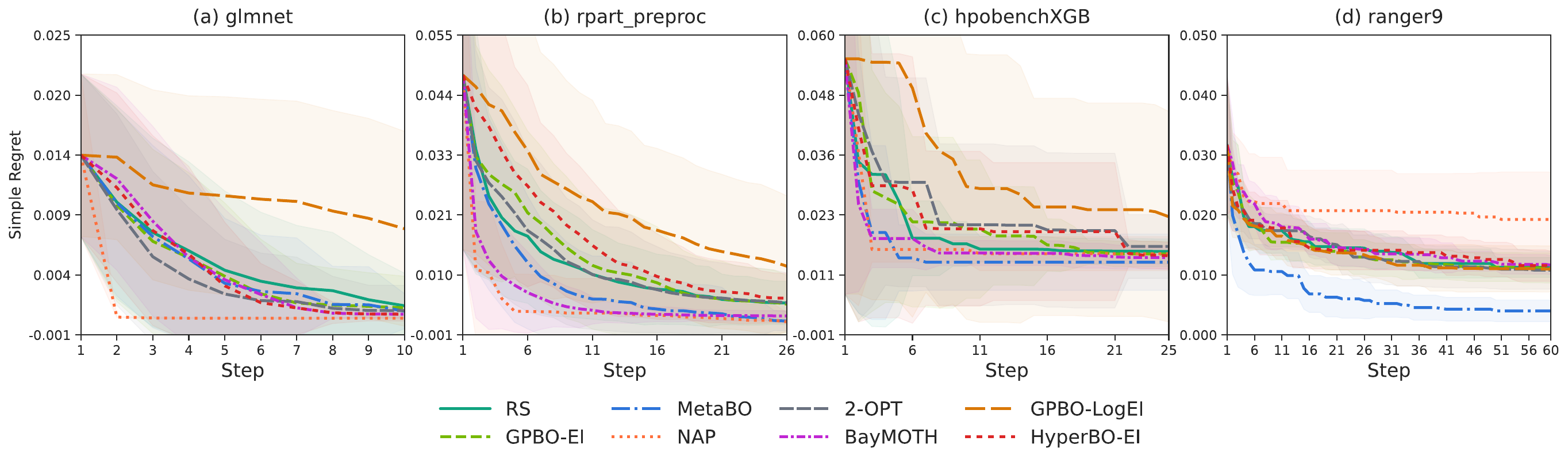}
    \caption{Simple regret trajectories on HPO tasks. The combined plot includes HPO-B tasks such as glmnet (ID:5860) and rpart.preproc (ID:4796), along with the HPOBench XGBoost results. Curves show performance versus sample steps, with mean trajectories and variability over randomly initialized starting locations.}
    \label{fig:hpo_all_results}
\end{figure*}

Here, we experiment on hyperparameter optimization tasks using the HPO-B benchmark \citep{arango2021hpo} and HPOBench dataset \citep{eggensperger2021hpobench}. These datasets contain accuracies of classification models as a function of model hyperparameters. The HPO-B datasets contains different search spaces comprising of model class, hyperparameter settings, and search domain which are evaluated on different supervised training datasets and splits. The HPO-B benchmark provides predefined datasets to be used for meta-training and meta-testing. From this benchmark dataset we select the search spaces glmnet(ID:5860, dim=2) and rpart.preproc(ID:4796, dim=3). For these search spaces all meta-training datasets are used for training and evaluation is done on one task drawn from the meta-test datasets. Due to resource limitations, for these HPO experiments, we restrict the number search spaces for evaluation. We refer the reader to \citep{arango2021hpo} for more details on this benchmark. Pertaining to HPOBench \citep{eggensperger2021hpobench} we use the dataset utilized by \cite{maraval2023end}, which is a hyperparameter optimization task (6D) for XGBoost. The same datasets as \cite{maraval2023end} are used for meta-training and testing is conducted on a dataset drawn from the test datasets.
\par On the glmnet search space, NAP substantially outperforms all competing methods, as shown in Figure~\ref{fig:hpo_all_results}. The other approaches, including BayMOTH, exhibit comparable performance. For the rpart search space, NAP again achieves the strongest results, followed by BayMOTH and MetaBO. Notably, a clear performance gap emerges between meta-BO methods and non-meta-BO baselines, underscoring the benefit of leveraging meta-information when it is available.
On the higher-dimensional HPOBench optimization task for XGBoost, meta-BO methods demonstrate a consistent advantage during the early optimization phase (up to step 5). Beyond this point, MetaBO ultimately attains the best overall performance.

\section{Discussion}

In this work, we introduce BayMOTH, a simple yet effective framework that unifies meta–Bayesian optimization and two-step lookahead BO. To the best of our knowledge, no prior work combines meta-BO with lookahead planning while explicitly accounting for the relatedness between source and target tasks and its impact on the robustness of the sampling policy. BayMOTH extends the 2-OPT acquisition function to operate with meta-information and also includes an immediate greedy improvement term. Both the horizon planning component and the immediate improvement component influence the sample suggestion strategy, as is analyzed in Appendix~\ref{app:ablation_alpha}. Despite the simplicity of BayMOTH, our experiments demonstrate that it is competitive with state-of-the-art meta-BO approaches. Moreover, by selectively incorporating meta-information, BayMOTH exhibits a robust sampling policy even when training tasks have limited similarity to the test task, a behavior that is observed across our experiments.
\par We acknowledge that BayMOTH’s interpretability and ease of implementation come at the cost of increased computational overhead. In particular, single-sample suggestion times can reach up to one minute in the 6-dimensional setting. More efficient optimization strategies \cite{wu2019practical} for the 2-OPT acquisition function could therefore be explored for the horizon-planning term in Equation~\ref{eq:BayMOTH} in future versions of the method.Nevertheless, in practical HPO settings, where per-evaluation costs are often orders of magnitude larger, this overhead is negligible. Moreover, our experiments suggest that BayMOTH often attains near-optimal configurations 2–3 evaluations earlier, making this trade-off favorable in practice. BayMOTH is particularly appealing in low- to moderately high-dimensional settings where source-task relatedness is uncertain and complicated training procedures are undesirable
\par In future works, we plan to investigate learned task-similarity identification mechanisms and more sophisticated routing strategies, potentially spanning multiple environments, to better exploit both global and local task features for improved sample efficiency. These directions may benefit from ideas drawn from mixture-of-experts (MoE) models \cite{mu2025comprehensive}. Additionally, we aim to further explore BayMOTH’s core principle, leveraging meta-information when it is informative and reverting to a robust default strategy otherwise, through alternative approaches such as Q-transformers \cite{chebotar2023q}, which are well suited for multi-task learning and could naturally integrate within a MoE framework.

\section*{Impact Statement}

The simplicity of our approach lowers the barrier to adoption and enables practitioners, especially those without specialized reinforcement learning expertise, to readily understand and fully leverage BayMOTH's capabilities. Overall, this paper presents work whose goal is to advance the field of sample efficient optimization. There are many potential societal consequences of our work, none which we feel must be specifically highlighted here.

\begin{ack}
This material is based upon work supported by the Department of Energy under Award Nos. DE-SC0024381, DESC0022132, DE-SC0021072 and DE-SC0024456 as well as the Department of Energy (National Nuclear Security Administration) University of Rochester “National Inertial Confinement Program” under Award No. DE-NA0004144. This report was prepared as an account of work sponsored by an agency of the United States Government. Neither the United States Government nor any agency thereof, nor any of their employees, makes any warranty, express or implied, or assumes any legal liability or responsibility for the accuracy, completeness, or usefulness of any information, apparatus, product, or process disclosed, or represents that its use
would not infringe privately owned rights. Reference herein to any specific commercial product, process, or service by trade name, trademark, manufacturer, or otherwise does not necessarily constitute or imply its endorsement, recommendation, or favoring by the United States Government or any agency thereof. The views and opinions of authors expressed herein do not necessarily state or reflect those of the United States Government or any agency thereof.
\end{ack}

\bibliography{BayMOTH}
\bibliographystyle{tmlr}

%%%%%%%%%%%%%%%%%%%%%%%%%%%%%%%%%%%%%%%%%%%%%%%%%%%%%%%%%%%%

\appendix
\newpage
\section{Appendix}
%%%%%%%%%%%%%%%%%%%%%%%%%%%%%%%%%%%%%%%%%%%%%%%%%%%%%%%%%%%%%%%%%%%%%%%%%%%%%%%
%%%%%%%%%%%%%%%%%%%%%%%%%%%%%%%%%%%%%%%%%%%%%%%%%%%%%%%%%%%%%%%%%%%%%%%%%%%%%%%
% APPENDIX
%%%%%%%%%%%%%%%%%%%%%%%%%%%%%%%%%%%%%%%%%%%%%%%%%%%%%%%%%%%%%%%%%%%%%%%%%%%%%%%
%%%%%%%%%%%%%%%%%%%%%%%%%%%%%%%%%%%%%%%%%%%%%%%%%%%%%%%%%%%%%%%%%%%%%%%%%%%%%%%

\subsection{BayMOTH Pseudo Code}
\label{app:baymoth_algo}

\begin{algorithm}[h]
\caption{\textsc{BayMOTH}}
\label{alg:baymoth}
\begin{algorithmic}[1]
\Require Meta-training datasets $\mathcal{D}_{\text{meta-train}}=\{\mathcal{D}_{\chi}\}_{\chi=1}^{\tau}$;
greedy balancing parameter $\alpha$; correlation threshold $\gamma$;
kernel $k(\cdot,\cdot)$ and $\mathcal{GP}$ hyperparameters $\theta$; unseen-task history $\mathcal{H}_0=\emptyset$;
domain $x \in D$;
optimization routine over $x$;
sample budget $T$.
\State \textbf{alternative virtual environments:} For each $\chi \in \{1,\ldots,\tau\}$ fit $\mathcal{GP}_{\chi}$ using $\mathcal{D}_{\chi}$.
\For{$t=1,2,\ldots,T$}
    \State \textbf{Fit target surrogate:} Fit $\mathcal{GP}_{t}$ on $\mathcal{H}_{t}=\{(x_i,f(x_i)\}_{i=1}^{t}$ using $(k,\theta)$.
    \State Compute predictive mean $\mu_{t}(\cdot)$ from $\mathcal{GP}_{t}$.
    \Statex

    \State \textbf{Construct virtual histories and select closest task:}
    \For{$\chi=1,2,\ldots,\tau$}
        \State Form virtual history $\tilde{\mathcal{H}}_{t}^{\chi}=\{(x_i,\tilde{f}_{\chi}(x_i)\}_{i=1}^{t}$ with $\tilde{f}_{\chi} = \mu_{\chi}(x_i)\text{ from } \mathcal{GP}_{\chi}$.
        \State Fit auxiliary $\tilde{\mathcal{GP}}_{\chi,t}$ on $\tilde{\mathcal{H}}_{t}^{\chi}$ using the same $(k,\theta)$.
        \State Compute predictive mean $\tilde{\mu}_{\chi,t}(\cdot)$ from $\tilde{\mathcal{GP}}_{\chi,t}$.
        \State Compute similarity score
        \[
        \mathcal{S}_{\chi,t} \leftarrow \mathrm{NCC}\!\left(\mu_{t},\tilde{\mu}_{\chi,t}\right)
        \]
    \EndFor

    \Statex
    \State \textbf{Sample selection:}
    \If{$\max_{\chi \in \{1,\ldots,\tau\}} \mathcal{S}_{\chi,t} > \gamma$} \Comment{correlation threshold $\gamma$}
        \vspace{1mm}
        \State $\chi^\star \leftarrow \arg\max_{\chi \in \{1,\ldots,\tau\}} \mathcal{S}_{\chi,t}$.
        \vspace{1mm}
       \State \textbf{Compute } $x_{t+1}$ \textbf{using optimization routine:}
        \State $x_{t+1} \leftarrow 
        \arg\max_{x_{t+1}}\Big[
        \Lambda_0(x_{t+1})
        +
        \,\mathbb{E}_{\mathcal{GP}_{\chi^\star}}
        \Big[
        (1-\alpha)\max_{x_{t+2}}\Lambda_1(x_{t+2}) + \alpha \Omega
        \Big]
        \Big]$ \Comment{$\pi_{\mathrm{BayMOTH}}$}

    \Else
        \State \textbf{Compute } $x_{t+1}$ \textbf{using optimization routine:}
        \State $x_{t+1} \leftarrow 
        \arg\max_{x_{t+1}}\Big[
        \Lambda_0(x_{t+1})
        +
        \,\mathbb{E}_{\mathcal{GP}_{t}}
        \Big[
        \max_{x_{t+2}}\Lambda_1(x_{t+2})
        \Big]
        \Big]$ \Comment{fallback}
    \EndIf
    \State Evaluate the true objective: $f_{t+1} \leftarrow f(x_{t+1})$.
    \State Update history: $\mathcal{H}_{t+1} \leftarrow \mathcal{H}_{t} \cup \{(x_{t+1},f_{t+1})\}$.

\EndFor
\end{algorithmic}
\end{algorithm}

\subsection{Proof of Properties of the Monte Carlo Estimator}
\label{app:proof}
Consider the second term in Eq.\ref{eq:BayMOTH}, also shown below
\[
\mathbb{J}(\omega) =  \mathbb{E}_{\omega \sim \mathcal{GP}_{\chi^{*}}}[\mathbb{Z}(\omega)],
\qquad
\mathbb{Z}(\omega) =(1-\alpha)\max_{x_{t+2}}\Lambda_1(x_{t+2};\omega) + \alpha \Omega(\omega)
\]

to approximate $\mathbb{J}$ we draw M iid samples $\omega^{(1)},\omega^{(2)},\dots,\omega^{(M)} \sim \mathcal{GP}_{\chi^{*}}$ and compute the estimator 
\[\hat{J}_M
\;=\;
\frac{1}{M}
\sum_{m=1}^M
\mathbb{Z}^{(m)},
\qquad
\mathbb{Z}^{(m)} = \mathbb{Z}(\omega^{(m)})
\]
The statistical guarantees for $\hat{J}_M$ are outlined below.

\noindent\textit{Estimator is unbiased:}\quad
\begin{align*}
\mathbb{E}[\hat{J}_M]
&= \mathbb{E}\!\left[\frac{1}{M}\sum_{m=1}^M \mathbb{Z}^{(m)}\right] \\[4pt]
&= \frac{1}{M}\sum_{m=1}^M \mathbb{E}[\mathbb{Z}^{(m)}]
&& \text{(linearity of expectation)} \\[6pt]
&= \frac{1}{M}\sum_{m=1}^M \mathbb{E}[\mathbb{Z}]
&& \text{($\mathbb{Z}^{(1)},\dots,\mathbb{Z}^{(M)}$ identically distributed)} \\[6pt]
&= \frac{1}{M}\, M\, \mathbb{E}[\mathbb{Z}] \\[4pt]
&= \mathbb{E}[\mathbb{Z}] \, .
\end{align*}

\noindent\textit{Estimator variance:}\quad
\begin{align*}
\mathbb{V}[\hat{J}_M] 
&= \mathbb{E}[\hat{J}_M^2] - \mathbb{E}[\hat{J}_M]^2 \\
&= \mathbb{E}\Big[\Big( \frac{1}{M}\sum_{m=1}^M \mathbb{Z}^{(m)}\Big)^2 \Big] - \mathbb{E}[\mathbb{Z}]^2\\
&= \mathbb{E}\Big[ \frac{1}{M^2} \sum_{m=1}^M (\mathbb{Z}^{(m)})^2 + \frac{1}{M^2} \sum_{m=1}^M\sum_{k=1,k\not=m}^M \mathbb{Z}^{(m)}\mathbb{Z}^{(k)} \Big] - \mathbb{E}[\mathbb{Z}]^2\\
&=  \frac{1}{M^2} \sum_{m=1}^M \mathbb{E}[(\mathbb{Z}^{(m)})^2] + \frac{1}{M^2} \sum_{m=1}^M\sum_{k=1,k\not=m}^M \mathbb{E}[\mathbb{Z}^{(m)}\mathbb{Z}^{(k)}]  - \mathbb{E}[\mathbb{Z}]^2\\
&=  \frac{1}{M^2} \sum_{m=1}^M \mathbb{E}[(\mathbb{Z}^{(m)})^2] + \frac{1}{M^2} \sum_{m=1}^M\sum_{k=1,k\not=m}^M \mathbb{E}[\mathbb{Z}^{(m)}]\mathbb{E}[\mathbb{Z}^{(k)}]  - \mathbb{E}[\mathbb{Z}]^2\\
&=  \frac{1}{M^2} \sum_{m=1}^M \mathbb{E}[\mathbb{Z}^2] + \frac{1}{M^2} \sum_{m=1}^M\sum_{k=1,k\not=m}^M \mathbb{E}[\mathbb{Z}]\mathbb{E}[\mathbb{Z}]  - \mathbb{E}[\mathbb{Z}]^2\\
&=  \frac{1}{M^2}M* \mathbb{E}[\mathbb{Z}^2] + \frac{1}{M^2} M(M-1)\mathbb{E}[\mathbb{Z}]^2  - \mathbb{E}[\mathbb{Z}]^2\\
&=  \frac{1}{M} \mathbb{E}[\mathbb{Z}^2] - \frac{1}{M}\mathbb{E}[\mathbb{Z}]^2\\
&=  \frac{\sigma^2}{M} \\
\end{align*}

\newpage
\subsection{Proofs}
\label{proof}
\begin{proof}[Proof sketch of Proposition \ref{regret_bound}]
Let \(\mathcal G_T\) denote the event that, for every round \(t \le T\), BayMOTH agrees with the oracle branch decision 

We first work on the event \(\mathcal G_T\).

Fix a round \(t\). Since \(g_t = g_t^\dagger\) on \(\mathcal G_T\), there are two cases.

If \(g_t^\dagger = 0\), both BayMOTH and the oracle comparator optimize the same fallback acquisition \(A_t^{\mathrm{2OPT}}\), and hence there is no acquisition discrepancy at round \(t\).

If \(g_t^\dagger = 1\), both procedures use the meta branch, and the discrepancy can be decomposed as
\[
\hat A_t^{\mathrm{BM}}(x) - A_t^\dagger(x)
=
\underbrace{
\bigl(\hat A_t^{\mathrm{BM}}(x) - A_t^{\mathrm{meta},\star}(x)\bigr)
}_{\text{Monte Carlo error}}
+
\underbrace{
\bigl(A_t^{\mathrm{meta},\star}(x) - A_t^{\mathrm{meta},\dagger}(x)\bigr)
}_{\text{virtual-environment mismatch}}.
\]

Under Assumption 3,
\[
\sup_{x \in \mathcal X_t}
\bigl|
A_t^{\mathrm{meta},\star}(x) - A_t^{\mathrm{meta},\dagger}(x)
\bigr|
\le \eta_t.
\]
For the Monte Carlo term, Assumption 4 and Chebyshev's inequality imply that for each fixed \(x \in \mathcal X_t\),
\[
\Pr\!\left(
|\hat J_{t,M}(x) - J_t^{\mathrm{meta},\star}(x)|
\ge
\frac{\sigma_t}{\sqrt{M\delta_t}}
\right)
\le \delta_t.
\]
Equivalently, for each fixed \(x \in \mathcal X_t\), with probability at least \(1-\delta_t\),
\[
|\hat J_{t,M}(x) - J_t^{\mathrm{meta},\star}(x)|
\le
\frac{\sigma_t}{\sqrt{M\delta_t}}.
\]

Now define the bad event
\[
\mathcal B_t
:=
\left\{
\exists x \in \mathcal X_t :
|\hat J_{t,M}(x) - J_t^{\mathrm{meta},\star}(x)|
>
\frac{\sigma_t}{\sqrt{M\delta_t}}
\right\}.
\]
By the union bound,
\[
\Pr(\mathcal B_t)
\le
\sum_{x \in \mathcal X_t}
\Pr\!\left(
|\hat J_{t,M}(x) - J_t^{\mathrm{meta},\star}(x)|
>
\frac{\sigma_t}{\sqrt{M\delta_t}}
\right)
\le
|\mathcal X_t|\,\delta_t.
\]
Hence, with probability at least \(1-|\mathcal X_t|\delta_t\),
\[
\forall x \in \mathcal X_t,
\qquad
|\hat J_{t,M}(x) - J_t^{\mathrm{meta},\star}(x)|
\le
\frac{\sigma_t}{\sqrt{M\delta_t}}.
\]
Since the Monte Carlo approximation enters through the acquisition, this implies
\[
\sup_{x \in \mathcal X_t}
\bigl|
\hat A_t^{\mathrm{BM}}(x) - A_t^{\mathrm{meta},\star}(x)
\bigr|
\le
\frac{\sigma_t}{\sqrt{M\delta_t}}.
\]

Therefore, on \(\mathcal G_T\),
\[
\sup_{x \in \mathcal X_t}
\bigl|
\hat A_t^{\mathrm{BM}}(x) - A_t^\dagger(x)
\bigr|
\le
\sup_{x \in \mathcal X_t}
\bigl|
\hat A_t^{\mathrm{BM}}(x) - A_t^{\mathrm{meta},\star}(x)
\bigr|
+
\sup_{x \in \mathcal X_t}
\bigl|
A_t^{\mathrm{meta},\star}(x) - A_t^{\mathrm{meta},\dagger}(x)
\bigr|
\le
\eta_t + \frac{\sigma_t}{\sqrt{M\delta_t}}.
\]
Multiplying by \(g_t^\dagger\) makes the bound valid uniformly over both fallback and meta rounds:
\[
\sup_{x \in \mathcal X_t}
\bigl|
\hat A_t^{\mathrm{BM}}(x) - A_t^\dagger(x)
\bigr|
\le
g_t^\dagger
\left(
\eta_t + \frac{\sigma_t}{\sqrt{M\delta_t}}
\right).
\]

Assumption 5 then transfers this to an objective-value difference:
\[
f(x_t^\dagger) - f(x_t^{\mathrm{BM}})
\le
g_t^\dagger
L_t
\left(
\eta_t + \frac{\sigma_t}{\sqrt{M\delta_t}}
\right).
\]
Since simple regret depends on the best value attained over all rounds, we obtain on \(\mathcal G_T\)
\[
r_T^{\mathrm{BM}}
\le
r_T^\dagger
+
\max_{1 \le t \le T}
g_t^\dagger
L_t
\left(
\eta_t + \frac{\sigma_t}{\sqrt{M\delta_t}}
\right).
\]

Finally, by Assumption 2 and a union bound over rounds,
\[
\Pr(\mathcal G_T^c)
\le
\sum_{t=1}^T |\mathcal X_t|\delta_t + \sum_{t=1}^T p_t,
\]
which gives the stated high-probability bound.

For the expectation bound, decompose over \(\mathcal G_T\) and \(\mathcal G_T^c\). On \(\mathcal G_T\), the above regret bound holds. On \(\mathcal G_T^c\), Assumption 1 gives the crude bound \(r_T^{\mathrm{BM}} \le B\). This yields
\[
\mathbb E[r_T^{\mathrm{BM}}]
\le
\mathbb E[r_T^\dagger]
+
\max_{1 \le t \le T}
g_t^\dagger
L_t
\left(
\eta_t + \frac{\sigma_t}{\sqrt{M\delta_t}}
\right)
+
B\left(
\sum_{t=1}^T |\mathcal X_t|\delta_t + \sum_{t=1}^T p_t
\right).
\]
\end{proof}

\newpage
\subsection{Ablation on $\gamma$ }
\begin{figure*}[h]
    \centering
    \begin{minipage}[t]{0.32\textwidth}
        \centering
        \includegraphics[width=\linewidth]{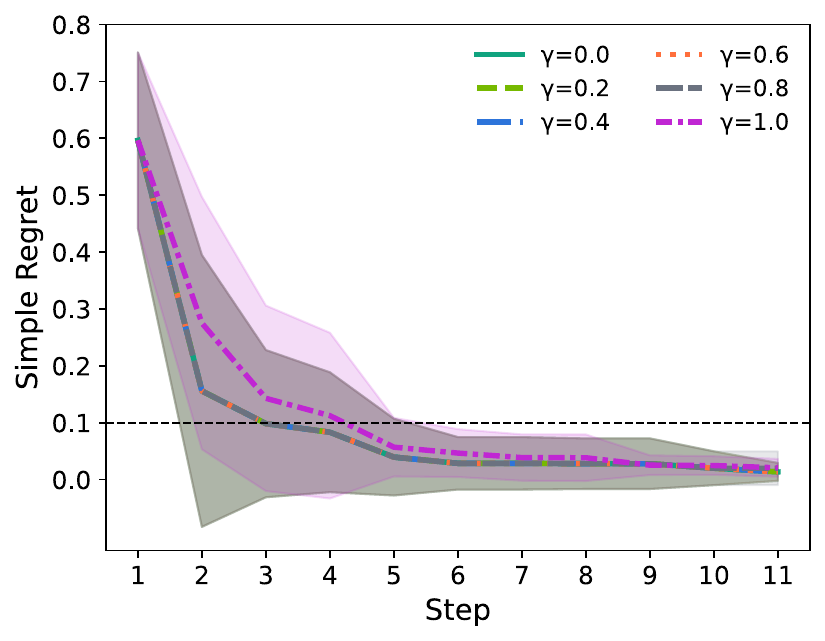}
    \end{minipage}\hfill
    \begin{minipage}[t]{0.64\textwidth}
        \centering
        \includegraphics[width=\linewidth]{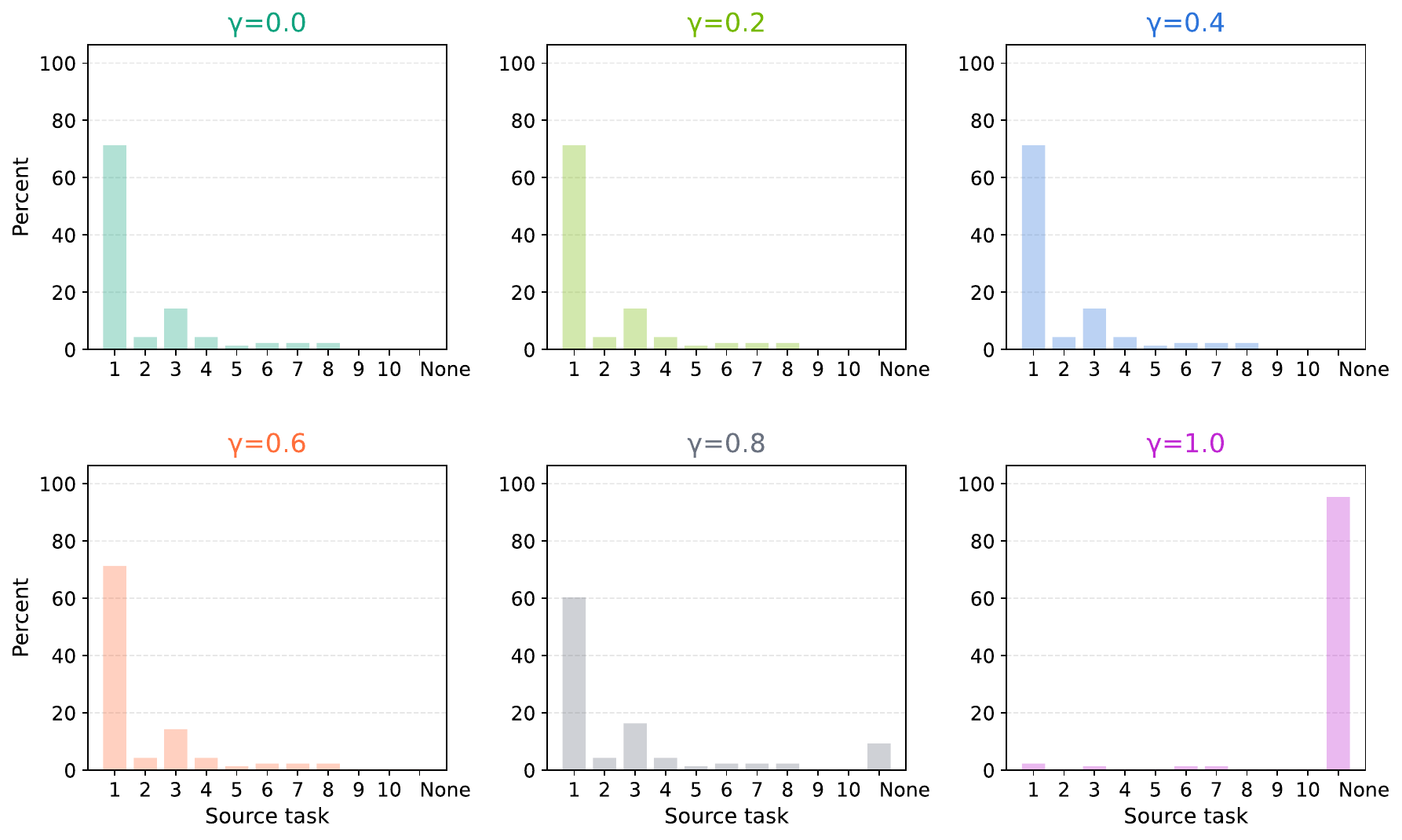}
    \end{minipage}
    \caption{Set 1 (highest NCC in set $\approx$ 0.85) trajectory and source-task usage across $\gamma$ values.}
    \label{fig:gamma-both}
\end{figure*}

\begin{figure*}[h]
    \centering
    \begin{minipage}[t]{0.32\textwidth}
        \centering
        \includegraphics[width=\linewidth]{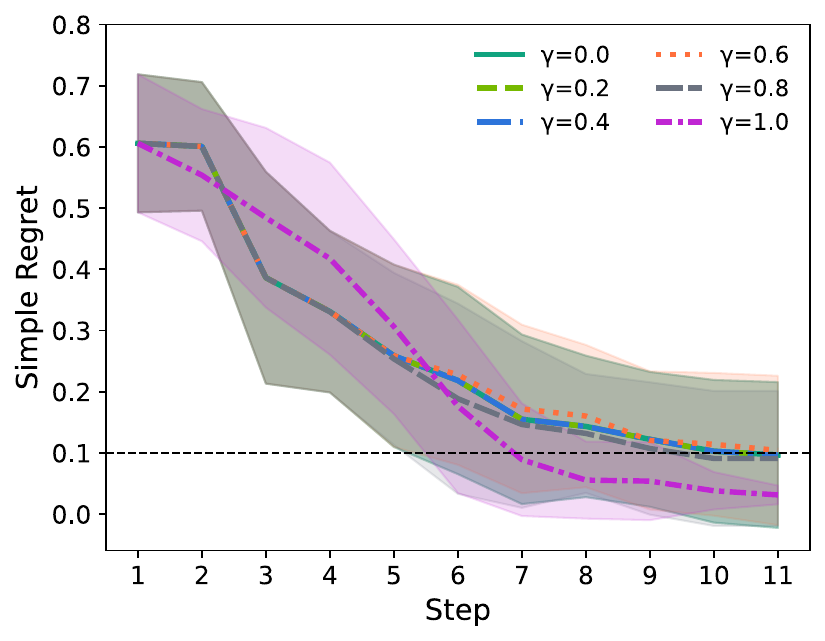}
    \end{minipage}\hfill
    \begin{minipage}[t]{0.64\textwidth}
        \centering
        \includegraphics[width=\linewidth]{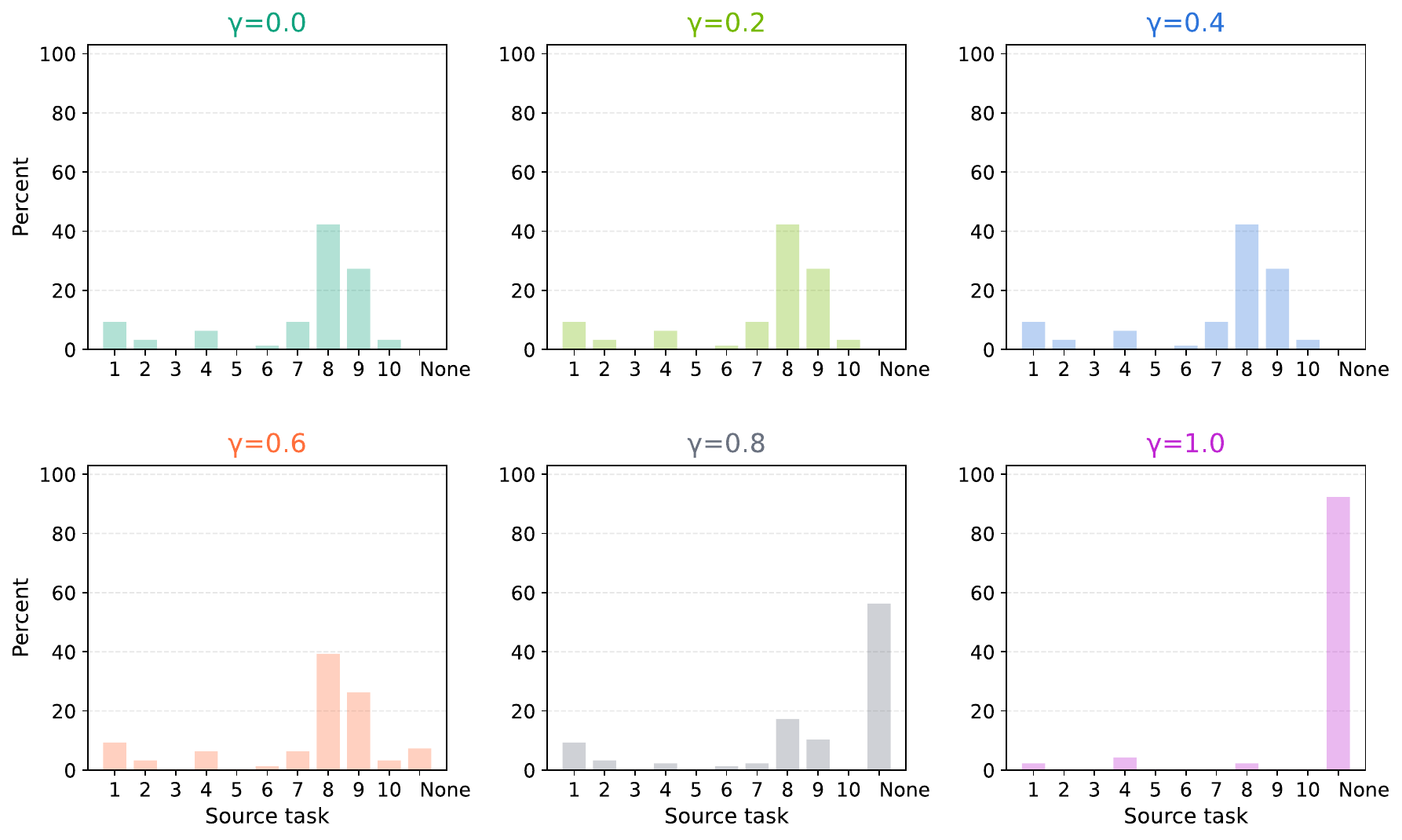}
    \end{minipage}
    \caption{Set 2 (highest NCC in set $\approx$ 0.60) trajectory and source-task usage across $\gamma$ values.}
    \label{fig:gamma-both}
\end{figure*}

\begin{figure*}[t]
    \centering
    \begin{minipage}[t]{0.32\textwidth}
        \centering
        \includegraphics[width=\linewidth]{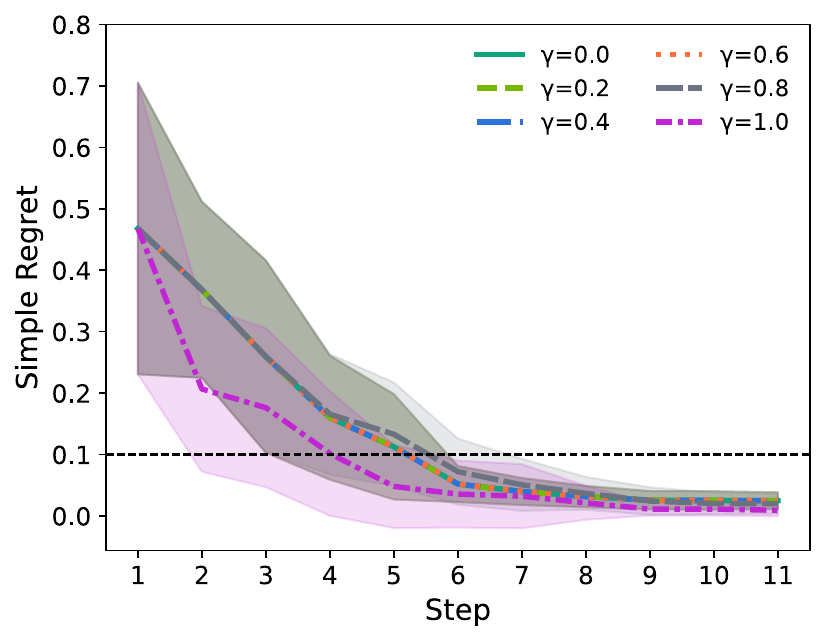}
    \end{minipage}\hfill
    \begin{minipage}[t]{0.64\textwidth}
        \centering
        \includegraphics[width=\linewidth]{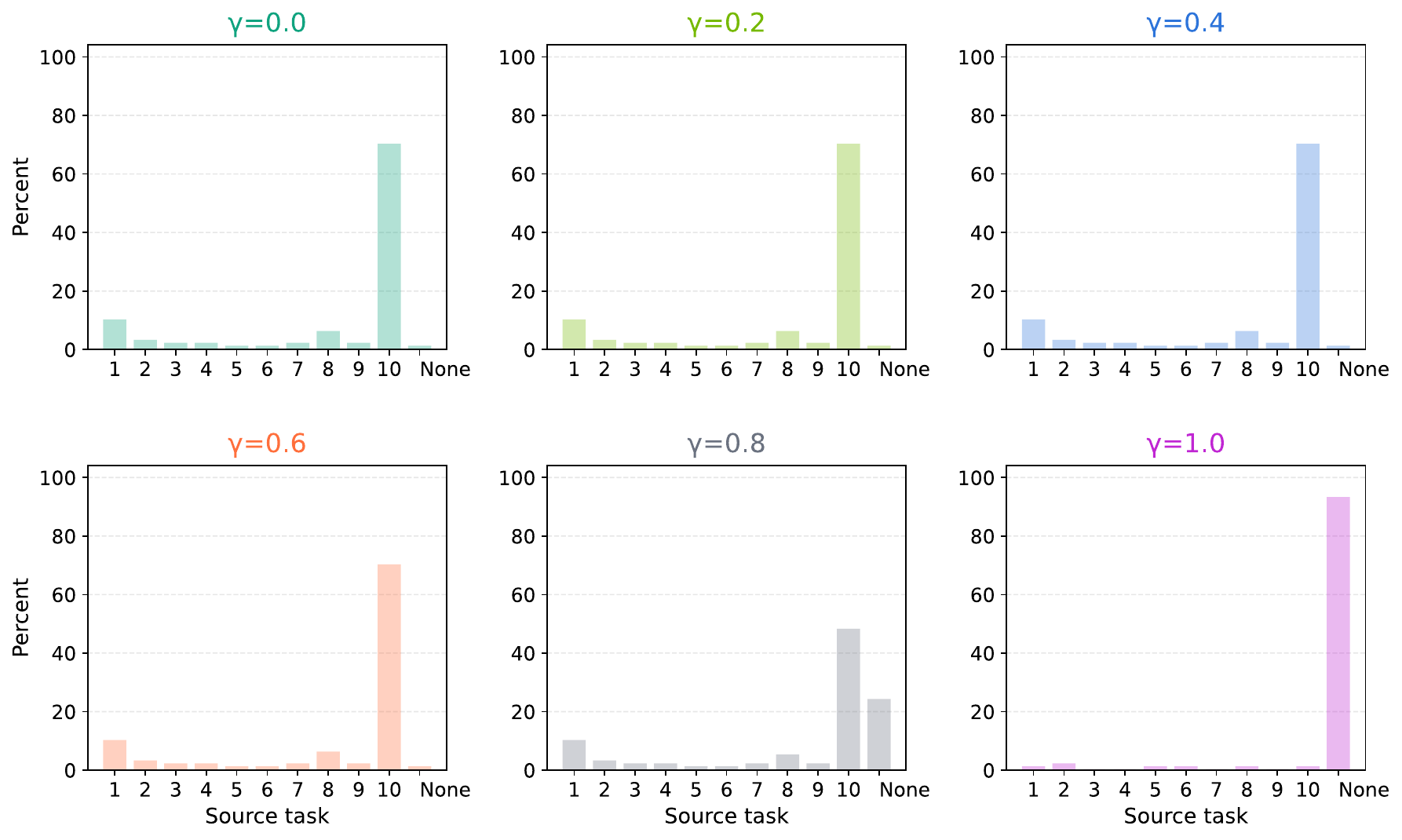}
    \end{minipage}
    \caption{Set 3 (highest NCC in set $\approx$ 0.82) trajectory and source-task usage across $\gamma$ values.}
    \label{fig:gamma-both}
\end{figure*}
\begin{figure*}[h]
    \centering
    \begin{minipage}[t]{0.32\textwidth}
        \centering
        \includegraphics[width=\linewidth]{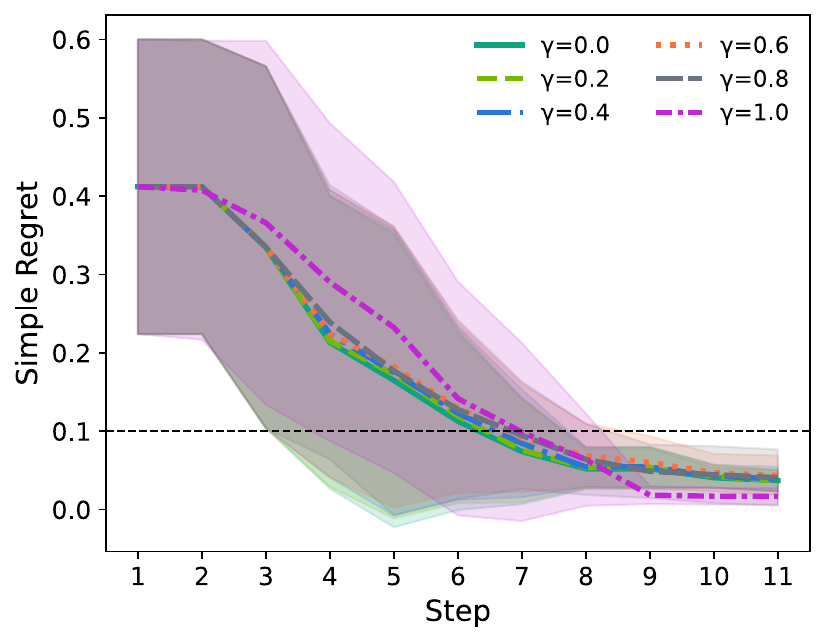}
    \end{minipage}\hfill
    \begin{minipage}[t]{0.64\textwidth}
        \centering
        \includegraphics[width=\linewidth]{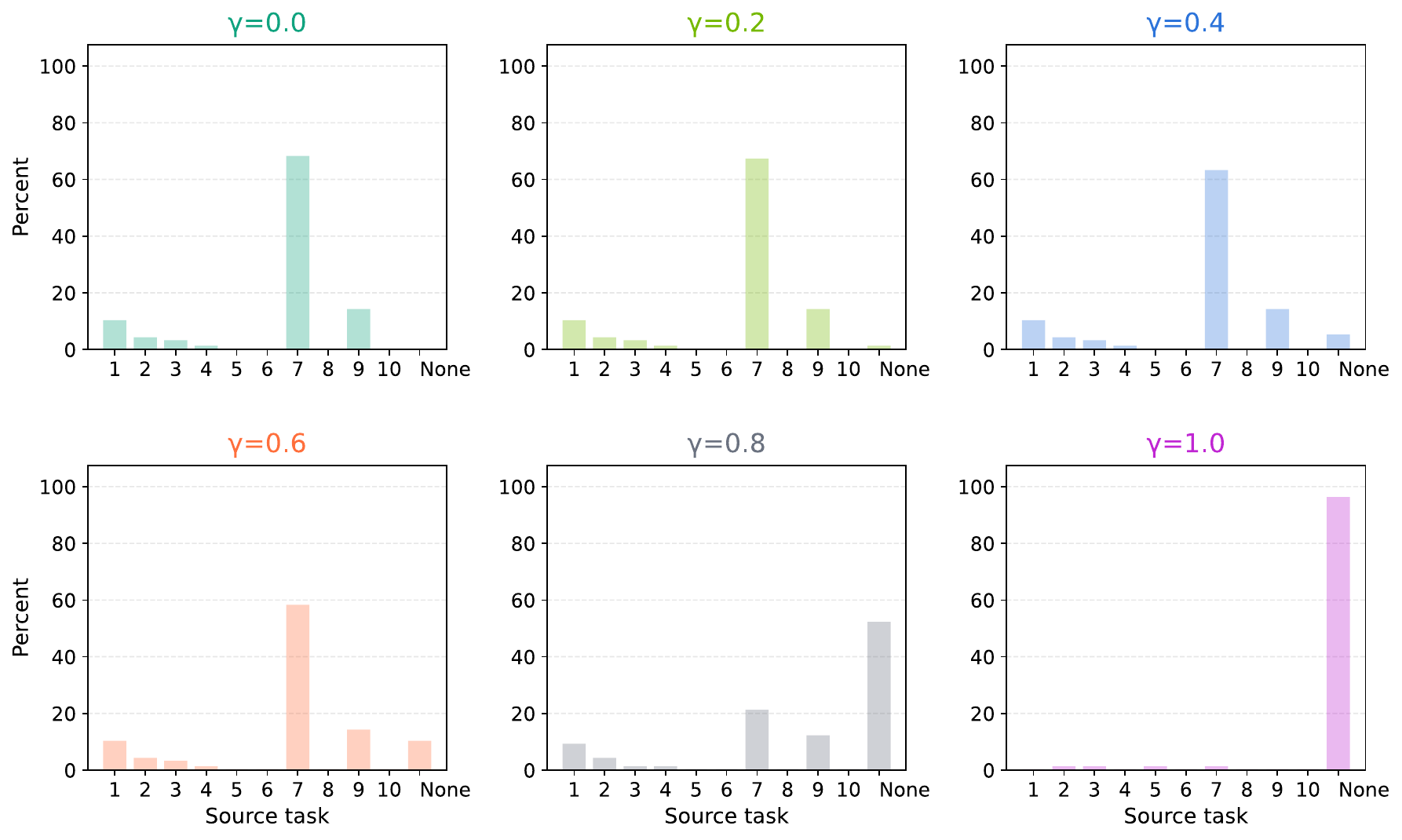}
    \end{minipage}
    \caption{Set 4 (highest NCC in set $\approx$ 0.10) trajectory and source-task usage across $\gamma$ values.}
    \label{fig:gamma-both}
\end{figure*}

This ablation studies the effect of the NCC threshold $\gamma$, which controls the extent to which BayMOTH relies on source task transfer, we sweep $\gamma$ and record the distribution of selected source tasks, measured as the percentage of total optimization steps, across randomly initialized runs. 

The results show that when $\gamma$ is below the NCC value of the most correlated source task, the optimization trajectory remains largely unchanged. This is consistent with the source task usage distributions, which are also nearly identical across these settings. Once $\gamma$ exceeds the NCC of the most correlated source task, however, the method enters a weak-transfer regime: usage of the highest-NCC source task drops substantially, and BayMOTH increasingly selects no source task as $\gamma \to 1$.

On the high-relatedness set, increasing $\gamma$ from $0.8$ to $1.0$ slows the initial reduction in regret, while leaving the final regret essentially unchanged. This behavior is consistent with the corresponding shift in source-task usage, from predominantly selecting source task 1 and source task 3 to predominantly selecting no source task. On the low-relatedness set, the $\gamma=1.0$ setting also yields a noticeably slower initial reduction in regret than the other threshold values, but achieves the best final performance, with regret approximately $2\%$ lower than the other settings. 

Overall, this ablation characterizes the behavior of BayMOTH with usage of source tasks across the spectrum from strong transfer to near always-fallback behavior.

\newpage
\subsection{Ablation on estimator number of samples M}
\label{app:ablation_m}

\begin{figure}[h]
    \centering
    \includegraphics[width=0.65\linewidth]{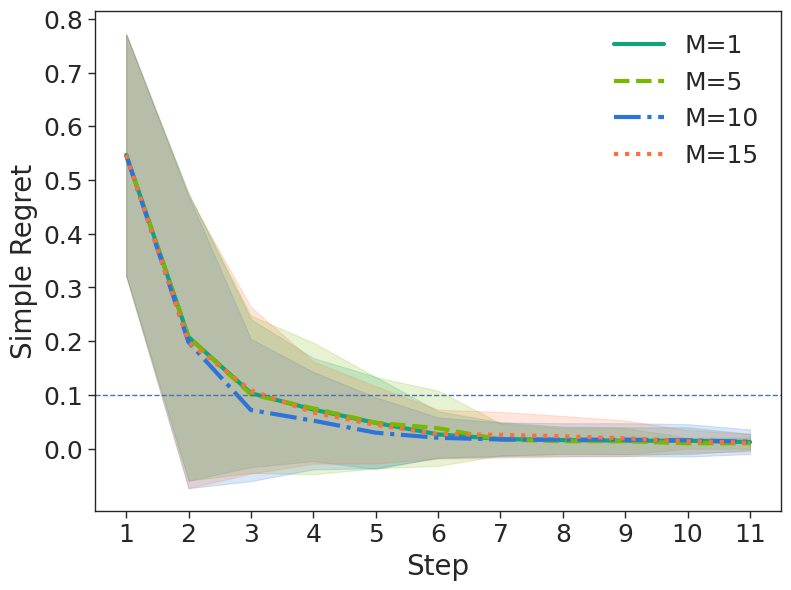}
    \caption{
        Scan on the number of samples for the Monte Carlo estimator, using set 1 of the synthetic ICF dataset. Shown is $<>$ line and $\pm\sigma$ shaded region over 30 randomly initialized starting locations.
    }
    \label{fig:m_number}
\end{figure}
Here, we investigate the effect of the sampling number for the Monte Carlo estimator of the expectation term in Equation~\ref{eq:BayMOTH}. We find that BayMOTH is only weakly sensitive to the choice of M on Set 1 of the synthetic ICF dataset. As is shown in Figure~\ref{fig:m_number}, increasing M beyond M=1 yields only marginal performance improvements, with M=10 performing the best, without a clear monotonic trend. This behavior is consistent with the chosen $\mathcal{GP}_{\chi^{*}}$ posterior being well-constrained in our experimental settings. Consequently, the estimator variance is small and a modest number of samples is sufficient. We anticipate that the dependence on M would become more pronounced in settings where the posterior variance of $\mathcal{GP}_{\chi^{*}}$ is larger, such as under severe data sparsity, higher-dimensional domains, or strongly misspecified kernels, where additional samples may be required to stabilize the estimator. We anticipate the M number to be treated as a hyperparameter which can be tuned using the training/validation tasks for a specific problem domain.

\subsection{Ablation on $\alpha$}
\label{app:ablation_alpha}

\begin{figure}[h]
    \centering
    \includegraphics[width=0.65\linewidth]{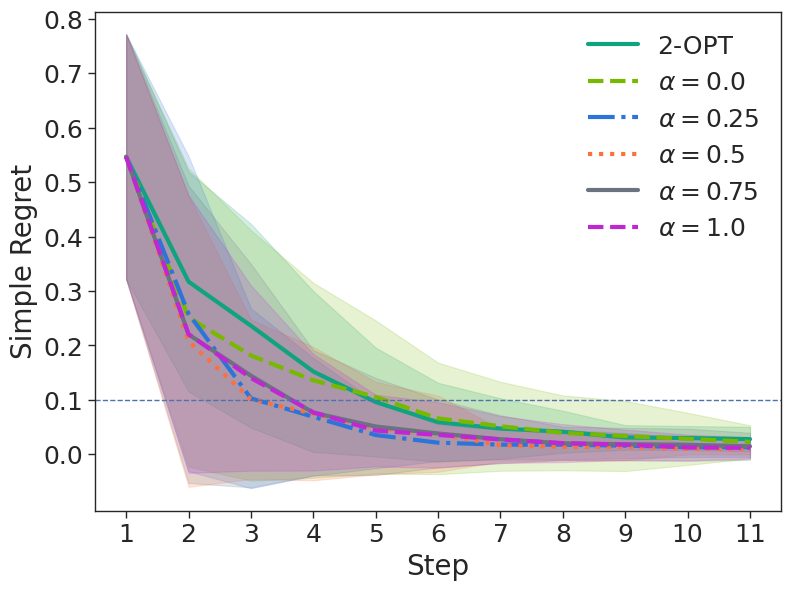}
    \caption{
        Ablation on the greedy balancing parameter $\alpha$. Shown is $<>$ line and $\pm\sigma$ shaded region over 30 randomly initialized starting locations.
    }
    \label{fig:appendix_example}
\end{figure}

Here, again using set 1 from the ICF datasets, we study the effect of the greedy balancing parameter $\alpha$ that controls the relative weighting between the immediate greedy improvement term and the meta-informed lookahead term in BayMOTH's policy. Several observations emerge. First, even for $\alpha=0$ which corresponds to pure horizon planning under $\mathcal{GP}_{\chi^{*}}$, BayMOTH outperforms the 2-OPT baseline during the initial 5 sample steps (approx. 10\% difference at sample step 3). In the domain of ICF experiments this is significant as sample evaluation costs are extremely high (500k USD - 1M USD per sample). This indicates that incorporating meta-information into $\Lambda_1$ provides a meaningful advantage in meta-BO settings even in the absence of the greedy improvement term. At the same time, intermediate values of $\alpha$ (specifically $\alpha \in \{0.25,0.5\}$) yield the strongest overall performance and exhibit comparable behavior, suggesting that jointly leveraging both immediate and lookahead components is beneficial. While $\alpha=1$, which places full emphasis on the immediate greedy improvement term, still improves upon $\alpha=0$, it slightly under-performs relative to intermediate settings. This pattern indicates that neither component alone is sufficient to fully capture the benefits of the policy; rather, their combination provides a more effective balance between short-horizon planning and meta-informed greedy exploitation. Therefore, the greedy balancing parameter should be treated as hyperparameter by a user and based on the nature of the problem domain and goals should be tuned.

% In the preamble:
% \usepackage{array}

\subsection{Computational cost analysis}
\label{app:compute-cost}

BayMOTH is not intended to reduce computational cost on a per-decision basis relative to offline-trained meta-BO methods. Instead, it makes a different computational tradeoff: rather than incurring the complexity of offline RL-based meta-training, it shifts computation to online planning within a simple and interpretable framework. Its intended advantage is therefore not lower proposal-time cost, but algorithmic simplicity, explicit decision logic, and the avoidance of a separate offline policy-training phase. In the settings considered in this paper, this tradeoff is favorable for the reasons discussed in the main text.

To make this tradeoff explicit, we compare the offline RL meta-training cost of MetaBO and NAP against BayMOTH's online cost on the \texttt{ranger9} test problem ($d=10$), and additionally provide a profiling-based breakdown of BayMOTH's internal runtime components.

In our implementation, the source-task GPs are fit once offline, while the comparison GPs used for virtual-environment selection are updated after each real observation. At proposal time, source-task selection itself is lightweight and requires only NCC scoring across candidate source tasks. The dominant computational cost does not arise from this GP-based selection mechanism, but from the nested two-step acquisition optimization used to generate each proposal. Specifically, each proposal is obtained using multi-restart L-BFGS-B optimization (3 restarts, up to 100 iterations per restart). Each acquisition evaluation then performs a horizon-step simulation, which, over $M$ Monte Carlo samples, requires future-GP fitting and future-EI evaluation. As a result, BayMOTH's runtime is driven primarily by the 2-OPT-style horizon-planning term rather than by the transfer-selection stage itself.

\begin{table}[t]
\centering
\caption{High-level computational cost profile of BayMOTH and offline-trained baselines.}
\label{tab:cost-profile}
\renewcommand{\arraystretch}{1.2}
\begin{tabular}{|>{\raggedright\arraybackslash}p{0.18\linewidth}
                |>{\raggedright\arraybackslash}p{0.70\linewidth}|}
\hline
Method & Cost profile \\
\hline
BayMOTH & Offline source task GP meta-training (cheap $\approx$ 40s); primary computational cost is paid online during proposal generation. \\
\hline
MetaBO & Main cost is offline PPO training; approximately 4 hours for 2000 PPO iterations. \\
\hline
NAP & Main cost is offline PPO training; approximately 11 hours for 2000 PPO iterations. \\
\hline
\end{tabular}
\renewcommand{\arraystretch}{1.0}
\end{table}

\vspace{5mm}

\begin{table*}[t]
\centering
\caption{Profiling-based breakdown of BayMOTH proposal-time cost on \texttt{ranger9} ($d=10$).}
\label{tab:baymoth-cost-breakdown}
\renewcommand{\arraystretch}{1.2}
\begin{tabular}{|>{\raggedright\arraybackslash}p{0.27\textwidth}
                |>{\raggedleft\arraybackslash}p{0.14\textwidth}
                |>{\raggedleft\arraybackslash}p{0.11\textwidth}
                |>{\raggedleft\arraybackslash}p{0.11\textwidth}
                |>{\raggedright\arraybackslash}p{0.29\textwidth}|}
\hline
BayMOTH stage & Frequency & Share of proposal time & Avg per call & Meaning \\
\hline
Source-task GP training
& once
& N/A
& 40.00\,s
& One-time offline training. \\
\hline
Comparison GP update
& once after observation
& N/A
& 0.234\,s
& Online overhead outside the proposer. \\
\hline
Proposal optimization
& once per suggestion
& 100.00\%
& 44.84\,s
& Full online proposal cost. \\
\hline
BayMOTH acquisition evaluation
& inside L-BFGS-B
& 98.66\%
& 145.85\,ms
& Nearly all proposal-time computation is spent here. \\
\hline
Horizon-step simulation
& once per acquisition evaluation
& 97.99\%
& 144.85\,ms
& Main bottleneck. \\
\hline
Future EI (inside horizon)
& $M$ times per acquisition evaluation
& 72.17\%
& 21.34\,ms
& Repeated future-EI computation dominates the inner loop. \\
\hline
Current EI (outside horizon)
& once per acquisition evaluation
& 0.67\%
& 0.99\,ms
& Negligible contribution. \\
\hline
Selection + NCC + optimizer orchestration
& once per suggestion
& 1.34\%
& 0.60\,s
& Upper bound on source selection and orchestration overhead. \\
\hline
\end{tabular}
\renewcommand{\arraystretch}{1.0}
\end{table*}

Overall, these results show that BayMOTH's computational overhead is dominated by explicit online lookahead planning, not by the GP-based source-task selection mechanism. In particular, the overwhelming majority of proposal-time cost is spent inside acquisition evaluation and horizon-step simulation, while source-task selection and orchestration contribute only a small fraction of runtime. This makes the computational tradeoff clear: BayMOTH avoids offline RL meta-training, but pays for this simplicity through a more expensive online planning step.

\noindent\textit{Computational resources.}
All experiments were conducted on a single workstation with an NVIDIA Quadro RTX 6000 GPU and an Intel(R) Xeon(R) W-2125 CPU @ 4.00GHz.

\section{Additional experimental details}
\label{app:exp_details}

\subsection{Synthetic ICF experiments}
\label{app:icf_exp_details}
\par For the synthetic ICF experiments, the optimization is done with respect to parameters $\text{x}1$ and $\text{x}2$. These parameters correspond to features of laser pulse shape design, which along with a target specification make up the explicit inputs ($\approx 20\text{D}$) for an ICF implosion. A number of performance metrics \cite{betti2010thermonuclear,betti2015alpha,christopherson2018theory,christopherson2020theory} have been devised to asses the quality of an implosion. The response of the performance metric with respect to $[\text{x}1,\text{x}2]$ is the focus of these optimization experiments. The parameters $[\text{x}1,\text{x}2]$ are consequential for performance through non-linear 1D and 3D physics effects, which are studied by using physics simulation codes such as LILAC \cite{delettrez1987effect}, DRACO \cite{radha2005multidimensional} and statistical models \cite{gopalaswamy2019tripled,lees2021experimentally,lees2025approaching}. However, as simplifying assumptions and reduced order modeling methodologies are employed in the simulations and models, the true response of performance w.r.t $\text{x}1$ and $\text{x}2$
is potentially not captured, which contributes to the lack of a priori predictive ability. The model inaccuracies could mean the value of their prior information could be limited, motivating the use of BayMOTH. 
\par For constructing source tasks, the physics models are varied in the physics simulator along with the coefficients of the 3D degradation model of \cite{lees2021experimentally} which is convolved with the simulation to produce a response function. This results in varied shock-timing \cite{craxton2015direct} from the front end of the pulse ($\text{x}1$ and $\text{x}2$ are part of the front end) resulting in a non-ideal variation of the entropy profile in the fusion fuel \cite{anderson2006adiabat} and consequently effecting the performance \cite{lees2025approaching}. Using this procedure, a total of 11 response functions were generated. We consider four distinct source-test task combinations for the experiments, these are shown in Figures \ref{fig:icf_all_sets}. The four source-test task configurations are constructed to exhibit progressively increasing distributional shift between the test task and the source tasks, with the lowest shift setting corresponding to the configuration with the highest average positive NCC, and subsequent settings defined by decreasing average positive correlation. These experiments are designed to span scenarios ranging from cases in which the source tasks (i.e., physics-based models) accurately capture the experimental response to regimes in which these models are increasingly misaligned with, or fail to represent, the true response, and as such these experiments probe the robustness of the BO methods when the assumed physics-based models progressively diverge from the true experimental response. Additionally, these experiments assess sample efficiency in realistic settings where a candidate physics-based response is provided as a source task, but its validity with respect to the true experimental response is unknown a priori.

\begin{figure}[p]
    \centering
    \includegraphics[width=\linewidth]{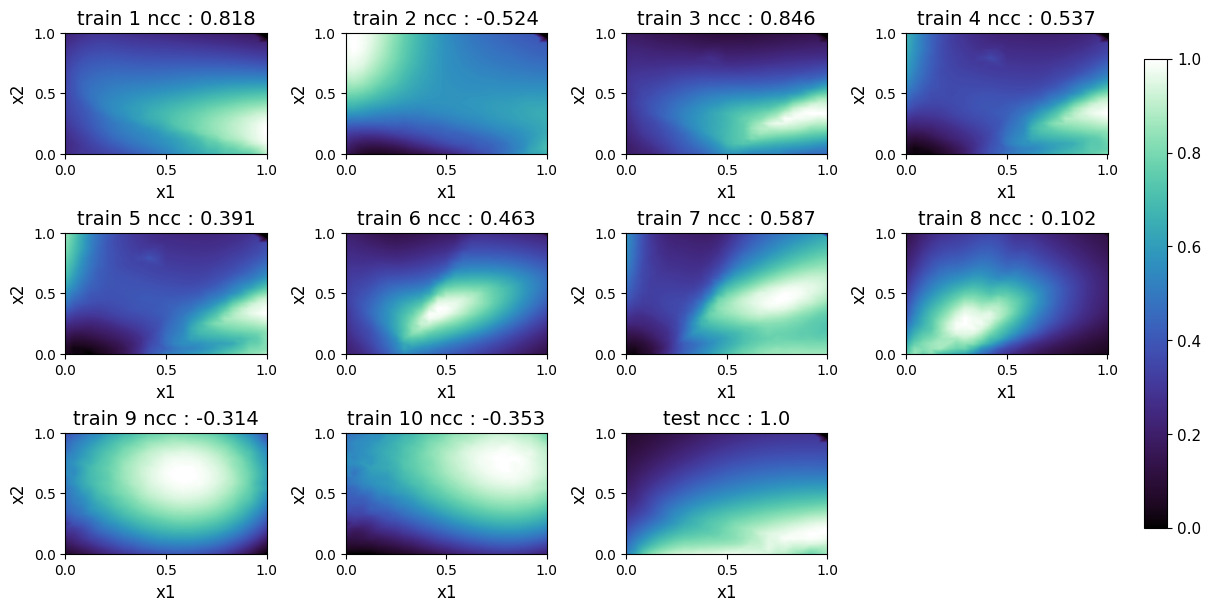}
    \caption*{(a) Set 1: High relatedness}

    \vspace{0.5em}

    \includegraphics[width=\linewidth]{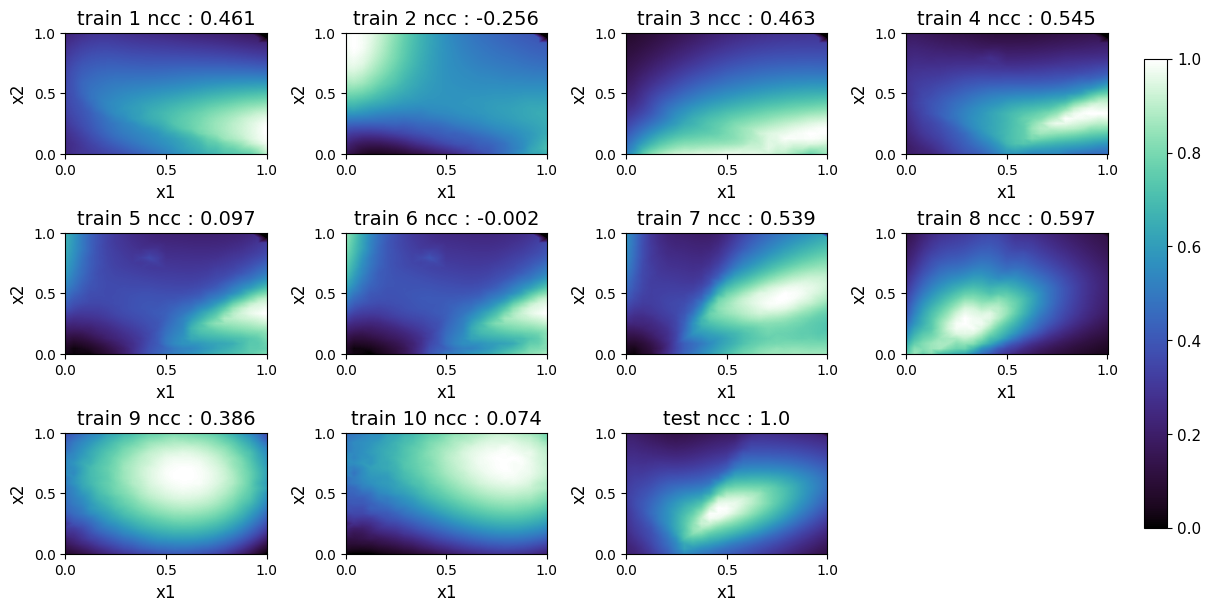}
    \caption*{(b) Set 2: Medium relatedness}

    \caption{Source/test tasks across relatedness settings. (a) Set 1: High relatedness. (b) Set 2: Medium relatedness. (c) Set 3: Moderate relatedness. (d) Set 4: Low relatedness.}
    \label{fig:icf_all_sets}
\end{figure}

\begin{figure}[p]\ContinuedFloat
    \centering
    \includegraphics[width=\linewidth]{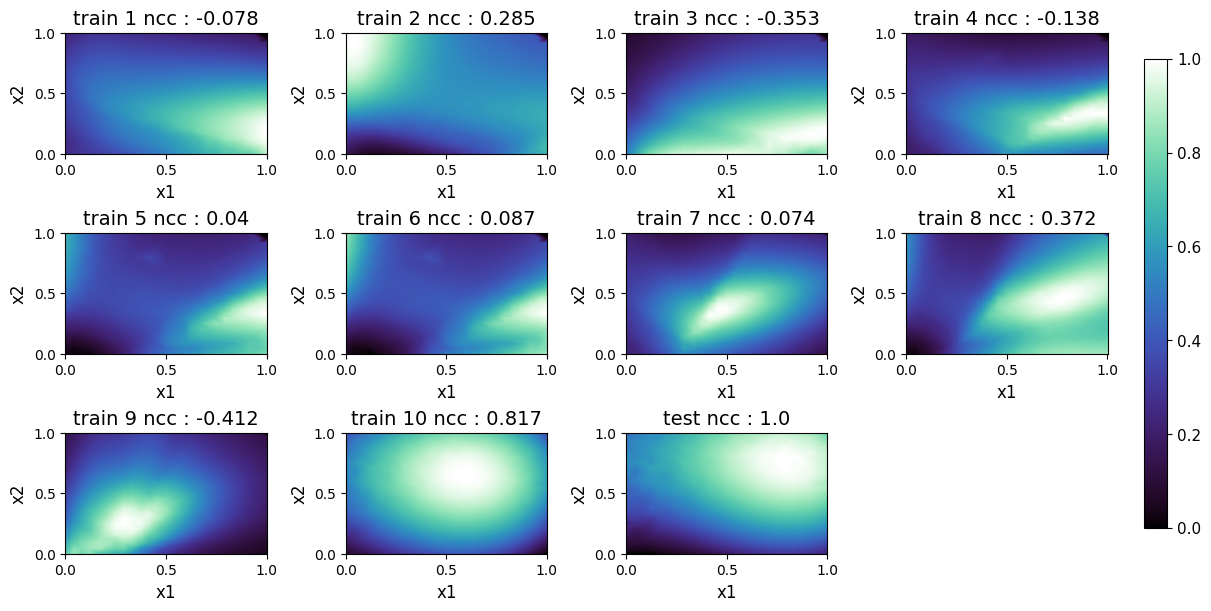}
    \caption*{(c) Set 3: Moderate relatedness}

    \vspace{0.5em}

    \includegraphics[width=\linewidth]{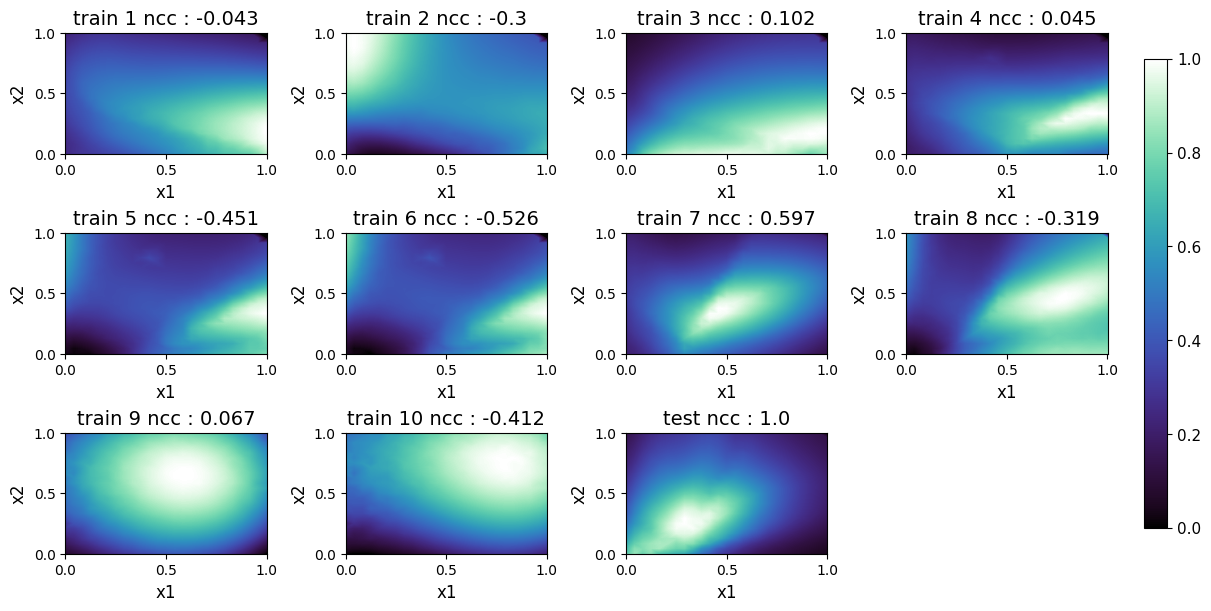}
    \caption*{(d) Set 4: Low relatedness}

    \caption[]{Source/test tasks across relatedness settings (continued).}
\end{figure}

\newpage
\subsection{Source Tasks Sensitivity}
\label{app:sensitivity_details}

\begin{figure}[t]
    \centering
    \includegraphics[width=\linewidth]{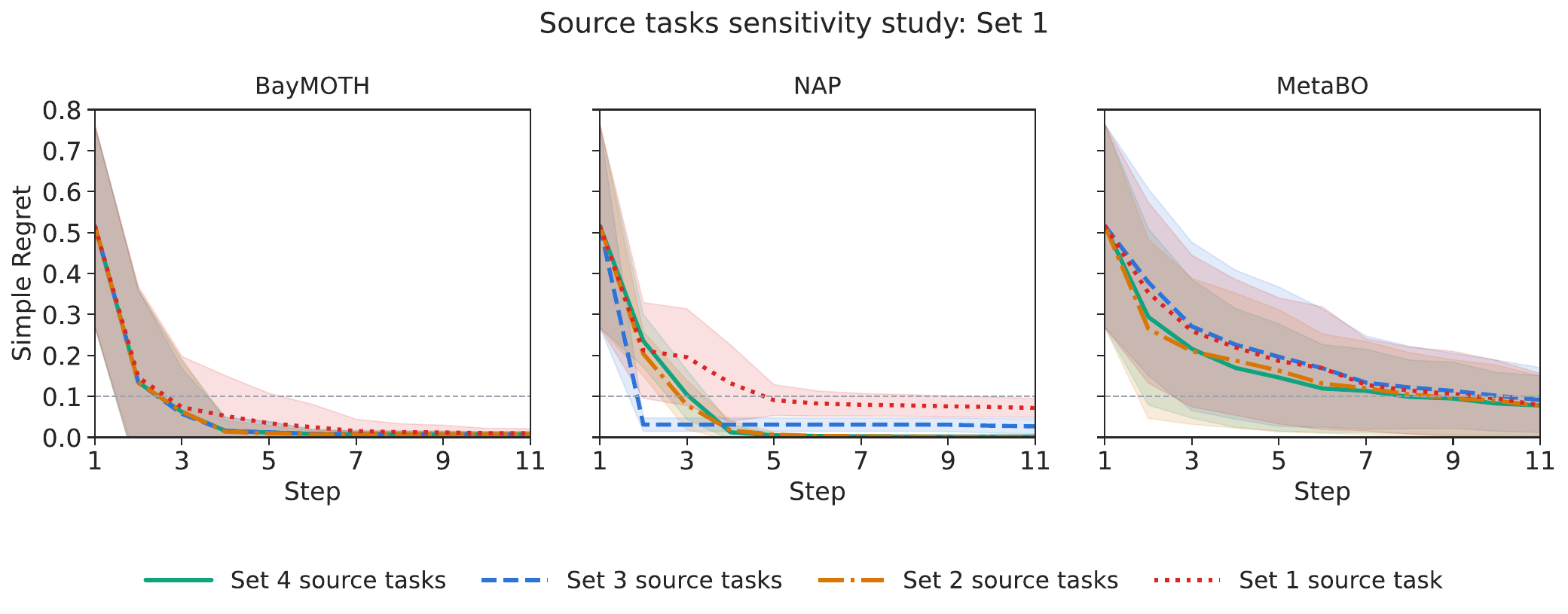}
    \caption{Source task pool sensitivity studies on test task performance (Set~1 test task). Shown is $<>$ line and $\pm\sigma$ shaded region over 100 randomly initialized starting locations.}
    \label{fig:sensitivity_set1}
\end{figure}

\begin{figure}[t]
    \centering
    \includegraphics[width=\linewidth]{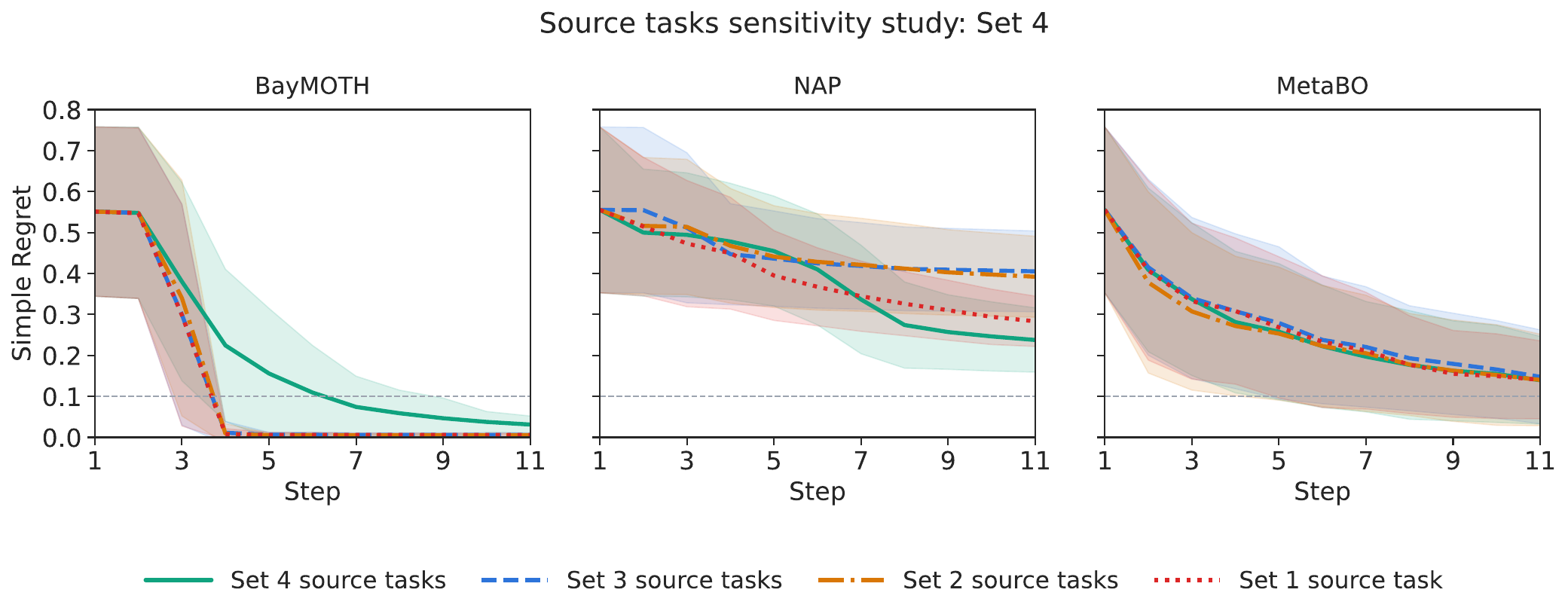}
    \caption{Source task pool sensitivity studies on test task performance (Set~4 test task). Shown is $<>$ line and $\pm\sigma$ shaded region over 100 randomly initialized starting locations.}
    \label{fig:sensitivity_set4}
\end{figure}

In Sections~\ref{icf_exps} and \ref{app:icf_exp_details} we describe the four source–test configurations (Sets 1–4), each specifying a test task and an associated pool of source tasks. Using those sets we additionally probe the sensitivity to the choice of source-task pool which is used for meta-training while holding the test task fixed for the meta-BO approaches, MetaBO, NAP and BayMOTH. We consider two target instances: (i) the test task from Set 1 is used and (ii) the test task from Set 4 is used. For each fixed target, we alternate the source-task pools from Sets 1–4, one pool at a time, without modifying their composition. Since these source pools are reused “as is,” there are three configurations in which the fixed test task appears among the selected source tasks i.e. when source tasks from set 2,3,4 are used in set 1 test task case and  when source tasks from set 1,2,3 are used in set 4 test task case. This setup enables us to understand how performance depends on source pool selection and if the meta-BO methods can tightly leverage the specific information needed (presented during meta-training) for enabling rapid optimization.
\par MetaBO exhibits low sensitivity to the choice of source pool, as evidenced in Figures~\ref{fig:sensitivity_set1} and~\ref{fig:sensitivity_set4}. The optimization trajectories show little variation across source-task configurations, suggesting robustness to source-task choice. However, it also indicates that MetaBO is less capable of learning sophisticated sampling policies that can accelerate optimization. NAP shows more variation in its trajectories indicating it is more sensitive to the choice of source pool. As seen in Figure~\ref{fig:sensitivity_set1}, in instances when the test task is including in the source pool, rapid reduction in regret occurs, indicating that NAP is effective in recognizing relevant meta-training information and utilizing it. However, Figure~\ref{fig:sensitivity_set4} shows that even when the test task is included in the source pool, NAP fails to reach the optimum. We hypothesize that this arises because the trajectory rollouts from NAP’s learned policy are dominated by the overall source pool distribution. In the case of using the set 4 test task, the average similarity between the source tasks and the test task remains low, and therefore, the learned policy generates poor samples, even if highly relevant meta-training information is present.
\par By contrast, BayMOTH maintains low regret across all source-pool configurations, with a discernible performance gain observed when the test task is included in the source pool. This behavior is expected as BayMOTH's protocol for gating to the pertinent alternative virtual environment allows it to make direct use of the relevant meta-information. In this setting, BayMOTH's performance exhibits minimal sensitivity to the choice of source pool, provided that the test task is included.

\subsection{Model Hyperparameters}

\begin{figure*}[!t]
    \centering
    \begin{subfigure}[t]{0.48\linewidth}
        \centering
        \includegraphics[width=\linewidth]{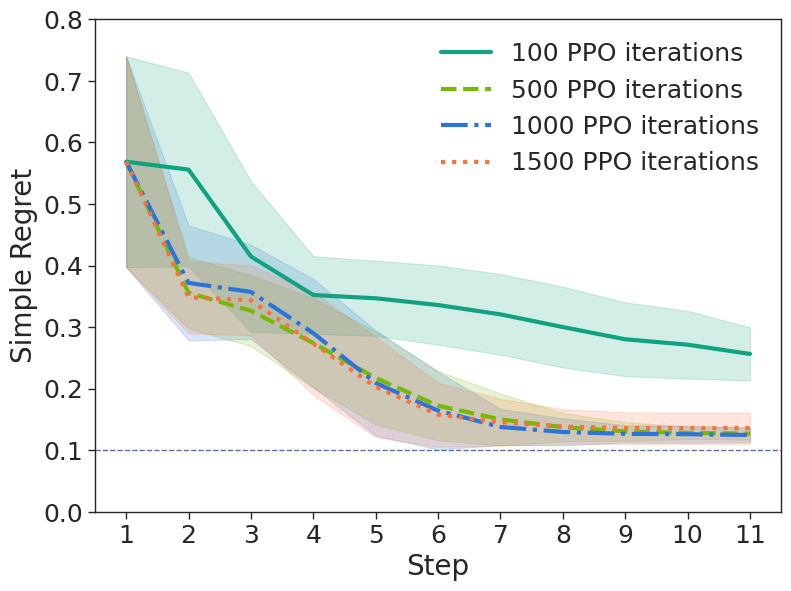}
        \caption{Convergence study over PPO iterations for NAP. Curves show the mean trajectory, with shaded regions indicating $\pm 1$ standard deviation over 100 randomly initialized starting locations.}
        \label{fig:app_nap_convergence}
    \end{subfigure}\hfill
    \begin{subfigure}[t]{0.48\linewidth}
        \centering
        \includegraphics[width=\linewidth]{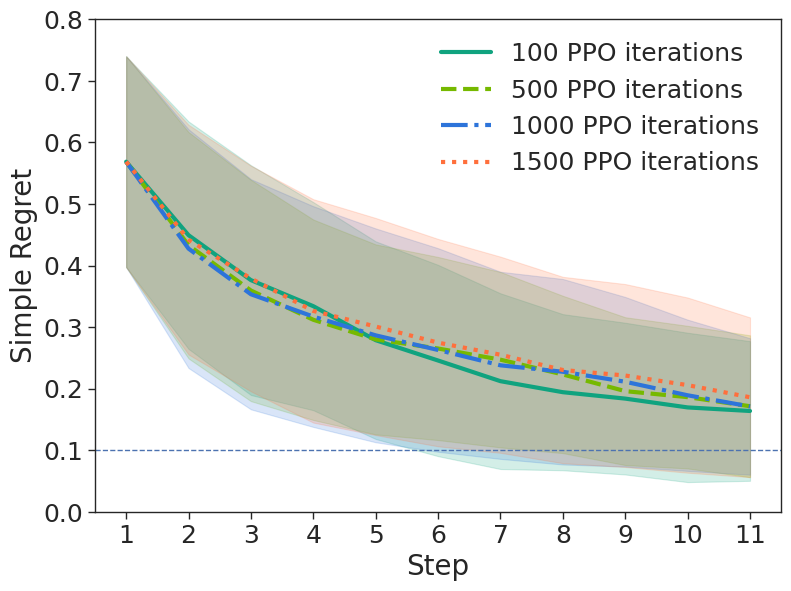}
        \caption{Convergence study over PPO iterations for MetaBO. Curves show the mean trajectory, with shaded regions indicating $\pm 1$ standard deviation over 100 randomly initialized starting locations.}
        \label{fig:app_metabo_convergence}
    \end{subfigure}
    \caption{Convergence study over PPO iterations using Set~2 of the ICF dataset.}
    \label{fig:convergence}
\end{figure*}

Here, we provide the associated hyperparameters used for the different models. Any additional details are present in the provided code repository. All hyperparameters were kept the same for all experiments. Minimal manual tuning was done for selection. The hyperparameters for MetaBO and NAP, were taken from their respective codebases, and the implementation from the NAP repository (https://github.com/huawei-noah/HEBO/tree/master/NAP) was used. The only change we make to the training configuration of MetaBO and NAP is for the synthetic ICF experiments, where we reduce the number of training PPO iterations to 500 due to computational resource limitations. However, we first conducted a convergence study on any changes to the optimization trajectory using set 2 of the ICF dataset. We found no meaningful difference between 500 PPO iterations vs more PPO iterations, we show evidence for this in Figure~\ref{fig:convergence}.

\begin{table}[h]
\centering
\caption{
BayMOTH configuration.
}
\label{tab:appendix_hyperparams_baymoth}
\begin{tabular}{l c}
\toprule
\multicolumn{2}{c}{\textbf{BayMOTH}} \\
\midrule
$\Lambda_0$ $\mathcal{GP}$ kernel                       & RBF \\
$\Lambda_0$ $\mathcal{GP}$ kernel  length-scale         & 0.05*4 \\

$\Lambda_1$ $\mathcal{GP}$ kernel                       & RBF \\
$\Lambda_1$ $\mathcal{GP}$ kernel  length-scale         & 0.05 \\

$\mathcal{GP}_{\chi}$ kernel                            & RBF \\
$\mathcal{GP}_{\chi}$ kernel  length-scale              & 0.05\\
$\tilde{\mathcal{GP}}_{\chi}$ Kernel                    & RBF\\
$\tilde{\mathcal{GP}}_{\chi}$ Kernel  length-scale      & 0.05*4\\
Estimator number of samples $M$                         & 5\\
greedy balancing parameter $\alpha$                     & 0.5\\

Observation noise  & $10^{-6}$ \\
Correlation threshold $\gamma$   & 0.7 \\
Input normalization        & $[0,1]^d$ \\
\bottomrule
\end{tabular}
\end{table}

\begin{table}[t]
\centering

\caption{2-OPT configuration.}
\label{tab:appendix_hyperparams_twoopt}
\begin{tabular}{l c}
\toprule
\multicolumn{2}{c}{\textbf{2-OPT}} \\
\midrule
$\Lambda_0$ $\mathcal{GP}$ kernel               & RBF \\
$\Lambda_0$ $\mathcal{GP}$ kernel length-scale  & 0.05*4 \\
$\Lambda_1$ $\mathcal{GP}$ kernel               & RBF \\
$\Lambda_1$ $\mathcal{GP}$ kernel length-scale  & 0.05 \\
Estimator number of samples $M$                         & 5\\
Input normalization                            & $[0,1]^d$ \\
\bottomrule
\end{tabular}

\vspace{2mm} % small, controlled gap

\caption{GPBO-EI configuration.}
\label{tab:appendix_hyperparams_gpbo_ei}
\begin{tabular}{l c}
\toprule
\multicolumn{2}{c}{\textbf{GPBO-EI}} \\
\midrule
$\Lambda_0$ $\mathcal{GP}$ kernel               & RBF \\
$\Lambda_0$ $\mathcal{GP}$ kernel length-scale  & 0.05*4 \\
Input normalization                            & $[0,1]^d$ \\
\bottomrule
\end{tabular}

\end{table}

%%%%%%%%%%%%%%%%%%%%%%%%%%%%%%%%%%%%%%%%%%%%%%%%%%%%%%%%%%%%

\newpage

\end{document}